\documentclass[12pt]{article}

\usepackage[margin=1in]{geometry}
\usepackage{setspace}
\usepackage{titlesec}
\usepackage{ragged2e}
\usepackage{parskip}
\usepackage{balance}
\usepackage{pdflscape}
\usepackage{fancyhdr}

\usepackage{graphicx,psfrag,epsf}
\usepackage{enumerate}
\usepackage{url}
\usepackage{comment}
\usepackage{abstract}
\usepackage{xr-hyper} 

\usepackage{longtable}
\usepackage{array}
\usepackage{booktabs}
\usepackage{multirow}

\usepackage{amsmath,amssymb}
\usepackage{bbm}
\usepackage{newtxtext}

\usepackage{algorithm}
\usepackage{algpseudocode}
\usepackage{listings}
\usepackage{xcolor}
\definecolor{burgundy}{RGB}{128,0,32}

\usepackage{float}
\usepackage[font=footnotesize,labelfont=bf,justification=raggedright,format=plain]{caption}
\usepackage[backend=biber,style=ieee]{biblatex}
\defbibheading{bibliography}[\refname]{%
  \section{\MakeUppercase{REFERENCES}}}

\usepackage{amsthm}
\usepackage{thmtools}
\usepackage{thm-restate}
\usepackage{hyperref} 

\newtheorem{proposition}{Proposition}

\newtheorem{lemma}{Lemma}
\newtheorem{assumptions}{Assumptions}
\newtheorem{definition}{Definition}
\newtheorem{remark}{Remark}

\newcommand{\blind}{0} 
\providecommand{\tightlist}{}

\newcommand{\R}{\mathbb{R}}
\newcommand{\PF}[1]{}
\DeclareMathOperator{\logit}{logit}

\DeclareMathOperator{\lauc}{LAUC}
\DeclareMathOperator{\auc}{AUC}

\newcommand{\idx}{\mathbbm{1}\left(f(x_k^+)>f(x_l^-)\right)}
\newcommand{\idX}{\mathbbm{1}\left(f(X^+)>f(X^-)\right)}
\newcommand{\1}{\mathbbm{1}}
\newcommand{\EF}[1]{\mathbb{E}_{X^+\sim F^+, X^-\sim F^-}\left[#1\right]}
\newcommand{\Eij}[3]{\mathbb{E}_{X_{#1}\sim P_{X_{#1}\in S_{#2}}}\left[#3\right]}

\newcommand{\X}{\mathcal{X}}
\newcommand{\Prob}{\mathbb{P}}

\newcommand{\ProbF}[1]{\mathbb{P}_{X^+\sim F^+, X^-\sim F^- }\left(#1\right)}

\newcommand{\olga}[1]{\textcolor{burgundy}{\textbf{[Olga: #1]}}}
\newcommand{\Wenye}[1]{\textcolor{brown}{\textbf{[Wenye: #1]}}}

\newcolumntype{L}[1]{>{\raggedright\arraybackslash}p{#1}}
\newcolumntype{C}[1]{>{\centering\arraybackslash}p{#1}}
\newcolumntype{R}[1]{>{\raggedleft\arraybackslash}p{#1}}

\begin{document}

\def\spacingset#1{\renewcommand{\baselinestretch}%
{#1}\small\normalsize} \spacingset{1}


\if0\blind
{
  \title{\textbf{Local Discrimination of Binary Classifiers via Localized AUC Functional}}
  
  \author{
    Wenye Ren\thanks{Contributed equally.}\textsuperscript{, 1, 2}, 
    Anna Győrffy-Kerekes\textsuperscript{*, 1, 3}, 
    Tancredi Dorsi\textsuperscript{4}, and 
    Olga Demler\textsuperscript{1, 5} \\[0.3cm]
    \small \textsuperscript{1} Department of Computer Science, ETH Zurich, Switzerland \\
    \small \textsuperscript{2} Department of Computer Science, Tsinghua University, China \\ 
    \small \textsuperscript{3} Center for Learning Systems, Max Planck Institute, Tübingen, Germany \\
    \small \textsuperscript{4} Department of Mathematics, ETH Zurich, Switzerland \\
    \small \textsuperscript{5} Mass General Brigham Hospital, Harvard Medical School, Boston, MA
  }
  \maketitle
} \fi

\if1\blind
{
  \bigskip
  \bigskip
  \bigskip
  \begin{center}
    {\LARGE\bf Title}
\end{center}
  \medskip
} \fi

\bigskip
\begin{abstract}
\subsection*{Background}
Binary classification models are widely used across diverse disciplines, with applications ranging from disease diagnosis and risk prediction to decision-making in finance and engineering. In clinical medicine, some of these models have been incorporated into national and international practice guidelines and inform treatment decisions for millions of patients worldwide. However,  model performance can vary substantially across patient subgroups, and tools to systematically identify these subgroups without loss of power are lacking.   

\subsection*{Methods}
In this work, we developed a method called LOCATOR that identifies high and low model performance regions in the covariate space without loss of power. 

We introduce a localized AUC functional LAUC, that, for a reference risk prediction model, estimates AUC for an index participant i using all participants weighted by similarity to i.  By using localized AUC functional our LOCATOR algorithm avoids subsetting data altogether.   Baseline phenotypic similarity between individuals is quantified using Gower distance.  We then train an XGBoost regression model to predict the logit-transformed LAUC using baseline characteristics and identified key drivers of heterogeneous model performance via SHAP (SHapley Additive exPlanations) analysis.

We establish relationship of localized AUC functional LAUC to the global AUC, showing that if similarity is flat, then its average is algebraically equal to global AUC. We construct estimator of LAUC and establish its asymptotic distribution and inference. We formulate conditions that guarantee that LAUC is appropriately ordered and derive its asymptotic distribution and standard error estimator. We show how localization identifies heterogeneity in classifier discrimination. We conduct simulations for both  positive and negative control scenarios.

 To facilitate inference and remediation of lower performance of the reference risk prediction model in the low competence region, we provide testing procedures. We show that, in addition to model heterogeneity, lower performance in the low competence region may be driven by the statistical characteristics of the validation cohort.

 As a practical application, we used LOCATOR to identify high and low performance regions of the PREVENT model in the National Health and Nutrition Examination Survey data (NHANES) 1999–2010 cycles (N = 17,748), identified high/low competence regions using our method, and validated high/low competence regions in the UK Biobank (N=355,122), an external cohort. 

\subsection*{Results}
The LOCATOR algorithm identifies high and low performance regions in the covariate space. SHAP analysis identified education, age, sex, current smoking, and prevalent hypertension as the most important determinants of AUC variation. Low and high performance subgroups were identified in training subset and applied to hold out test subset of NHANES. Performance in low subgroup ranged from .6 to .75 and from .82 to .9 in high performance subgroup . When splitting at the median LAUC, conventional AUC values in the low performance subgroup were .78 [.69; .77] and .88 [.82; .93] in the high performance subgroup in the hold-out subset of NHANES. In UK Biobank, an external cohort, AUC of PREVENT model in these regions produce AUC of .76 [.74; .78] and .71 [.70; .72]. 

\subsection*{Conclusions}
The LOCATOR approach offers a data-driven strategy to identify subgroups with high and low performance of the existing binary classification model without loss of power. In the low performance subgroup, remediation strategies include model refitting, enriching the covariate space with new biomarkers, and identification and down weighting of regions with low variance of predicted PREVENT risk. The proposed localized AUC functional, LAUC, offers an objective metric that can be reported for participants in the high or low performance subgroups as measure of the strength of the reference model.

\end{abstract}

\noindent%
{\it Keywords:}  Model discrimination, AUC, AUROC, c-statistics, subgroup identification, binary classification model, logistic regression, model heterogeneity.
\vfill

\newpage
\spacingset{1.45} 
\section{Introduction}
\label{sec:intro}

Binary classification models are some of the most widely used models. Their applications range from detection/diagnosis to prospective risk stratification. Therefore, an accurate and nuanced assessment of their quality of performance is of paramount importance. Yet, the Area Under the Receiver Operating Characteristics Curve \cite{AUC} (AUROC, also known as AUC, c-index, concordance index, Harrell or Uno c-index), a widely used measure of model performance, rarely exceeds 0.80 across guidelines-recommended models (Khan et al., Circulation 2024 \cite{PREVENT-dev}). It is well below its theoretical maximum of 1.0 and has room to grow. However, current methods cannot identify, without loss of power, where this room is. The focus of this paper is to develop  reproducible methods that identify, without loss of power, where the deployed classifier performs well or not so well, and separate genuine heterogeneity from cohort artifacts to guide model improvement in low-competence regions. 

Several treatment guidelines rely on official binary classifiers  that inform risk stratification for millions of people. Systematic COronary Risk Evaluation 2 (SCORE2) \cite{SCORE2} and Predicting Risk of Cardiovascular Disease EVENTs (PREVENT) \cite{PREVENT} models are embedded in European and US prevention guidelines, respectively, and supply the risk estimates that decide lipid- and blood-pressure-lowering therapy for hundreds of millions of adults \cite{PREVENT}. Before release, these models are validated in external cohorts by reporting discrimination, calibration, and reclassification statistics. AUROC, a measure of a model's ability to separate events from non-events, is a ubiquitous measure of model discrimination. It is often reported as one number per external cohort. However, a single per-cohort number cannot rule out the possibility that the model may separate events from non-events well in one region of covariate space and no better than chance in another. Patients in the weak regions receive systematically less accurate risk estimates and, therefore, less appropriate treatment. Indeed, accumulating evidence suggests that the performance of risk prediction models, including PREVENT, is not uniform across all patient subgroups \cite{heteoroPREVENT}. Traditional approaches to investigating these performance variations typically rely on predefined, dichotomous subgroup classifications (e.g., smokers vs. non-smokers) \cite{risk_hetero_analysis_method}. While informative, such methods rely on subsetting the data and may fail to capture more complex, multidimensional patterns of phenotypic heterogeneity due to quick loss of power. There is no established estimand for local discrimination without subsetting the data, so heterogeneity is difficult to locate, quantify, or test without quickly loosing power.

In parallel, the concept of phenomapping has recently emerged \cite{oikonomou}. Developed to assess heterogeneity of treatment effect in randomized controlled trials, it leverages the full breadth of data.   By quantifying inter-individual similarity based on multiple clinical features (e.g., using Gower distance \cite{gower}), phenomapping allows for the discovery of contiguous, data-driven patient subgroups with heterogeneous treatment effect. This approach has been successfully applied to identify heterogeneity of treatment effects in clinical trials \cite{oikonomou}. It relies on estimating the average per-person hazard, and therefore cannot be directly translated to AUROC. In this paper, we introduce a localized AUC functional LAUC, that, for a reference risk prediction model, estimates AUC for an index participant i using all participants weighted by similarity to i. By using the localized AUC functional, our LOCATOR algorithm avoids subsetting data altogether.  

Then we use localized AUC functional (LAUC) to identify contiguous phenotypic subgroups in which the reference binary classification model exhibits suboptimal discrimination. We further seek to characterize the features that drive this performance heterogeneity. 

We establish the relationship of localized AUC functional LAUC to the conventional global AUC. We provide two sufficient conditions and prove that they guaranty the LOCATOR algorithm will identify high and low performance regions in the covariate space. We construct an estimator of LAUC and establish its asymptotic distribution and a closed-form estimator for the standard error. 

 To facilitate inference and remediation of lower performance of the reference risk prediction model in the low LOCATOR region, we provide testing procedures. We show that, in addition to model heterogeneity, lower performance in the low competence region may be driven by the statistical characteristics of the validation cohort.

Our LOCATOR algorithm provides a novel, data-driven strategy for targeted model refinement while avoiding loss of power. It also provides a more nuanced strategy for model refinement, showing that it is an interplay of model heterogeneity, validation cohort characteristics, and overlap in predicted risk in certain regions of covariate space.
\newpage

\section{Properties of The Localized AUC and Asymptotic Theory}
\label{sec:theory}

\subsection{Definition of Localized AUC Functional}
\label{sec:def_local_auc}

The localized AUC is a modified version of the well-known AUC. Its main purpose is to give a personalized performance of the risk predictor. 

\begin{definition}[Localized AUC Functional]
\label{def:local_auc}

Let $\mathcal{X}$ denote covariate space. If we sample a patient with an event and covariates $X^+\sim F^+$, where $F^+$ is the sampling distribution of our event patients, and a patient with no event with covariates  $X^-\sim F^-$. Let $x_i\in\mathcal{X}$ be covariates for our index patient $i$  . We also have a risk prediction model $f:\mathcal{X}\to[0,1]$. In addition, we also have a weight function $w:\mathcal{X}\times\mathcal{X}\to[0,1]$, that describes the similarity between patients. Then the localized AUC is as follows:
\begin{align}
    \lauc(x_i)=\frac{\EF{w(x_i,X^+)w(x_i,X^-)\idX}}{\EF{w(x_i,X^+)w(x_i,X^-)}}
\end{align}
\end{definition}
\begin{definition}[Localized AUC Estimator]
\label{def:local_auc_estimator}
Let's say our dataset includes $K$ patients with events with covariates $x_1^+, x_2^+, \dots, x_K^+$ and we also have $L$ patients that have no events with covariates $x_1^-, x_2^-, \dots, x_L^-$. We also have a risk prediction model $f:\mathcal{X}\to[0,1]$. In addition, we also have a weight function $w:\mathcal{X}\times\mathcal{X}\to[0,1]$. Then the localized AUC estimator for some sample patient $x_i\in\mathcal{X}$ is
\begin{align}
\widehat{\lauc}(x_i)=\frac{\sum_{k=1}^K\sum_{l=1}^L w(x_i, x_k^+)w(x_i,x_l^-)\idx}{\sum_{k=1}^K\sum_{l=1}^L w(x_i, x_k^+)w(x_i,x_l^-)}
\end{align}
\end{definition}

In our paper the weight function will be a kernel, that is based on the Gower distance.

\begin{definition}[Gower distance \cite{gower}]
\label{def:gowers_distance}
    Let $x,y\in\mathcal{X}$, where $\mathcal{X}$ is a $P$ dimensional covariate space. Suppose that $\forall p=1,\dots,P$, the ranges of $x_p$ and $y_p$ are both $R(p)$. Then the Gower distance is
    \begin{align}
        d(x,y)=\sum_{p=1}^{P} 1 - \frac{|x_p-y_p|}{R(p)}.
    \end{align}
\end{definition}

In our paper, the weight function will be $w(x,y)=(1-d(x,y))^{15}$.

\subsection{ Theoretical Properties of Localized AUC Functional}
\label{sec:theoretical_properties}

\subsubsection{Relationship of \texorpdfstring{$\lauc(x_i)$}{AUC	extsubscript{i}} to global AUC under uniform similarity}
\label{sec:rel_auc_vs_local_auc}

If the weights are uniform, we end up with the traditional AUC. 

\begin{restatable}[Under Uniform Weights the Localized AUC Reduces to the AUC]{proposition}{underuniform}
\label{prop:AUCiunderuniform}
Let $\mathcal{X}$ denote the covariate space. If we sample a patient with an event and covariates $X^+\sim F^+$, where $F^+$ is the sampling distribution of our event patients, and a patient with no event with covariates $X^-\sim F^-$. Let $x_i\in\mathcal{X}$ be covariates for our index patient $i$. We also have a risk prediction model $f:\mathcal{X}\to[0,1]$. In addition, we also have a weight function $w:\mathcal{X}\times\mathcal{X}\to[0,1]$, that is some constant $c$. Then the localized AUC is equivalent to the traditional AUC:
\begin{equation}
    \lauc(x_i)=\auc.
\end{equation}
\end{restatable}

The proof can be found in the Supplementary material \ref{proof:AUCiunderuniform}

\subsubsection{Localized AUC as Covariance}
\label{sec:LAUC_as_covariance}

The weighted AUC has the property as Proposition \ref{prop:AUCcov}:

\begin{restatable}[Localized AUC as Covariance]{proposition}{AUCcov}
\label{prop:AUCcov}
    Suppose number of events is $n$ and number of non-events is $m$.
    Define $W_{kl} = w^+_k w^-_l$ where $w^+$ is the weight vector of event participants, and $w^-$ is the weight vector of non-event participants. Then $W\in R^{n\times m}$
    With two types of weight vectors $w_i, w_j$, The $W^{(i)}, W^{(j)}$ can be calculated. 
    For event participants $k, a$ and non-event participants $l,b$, define $I_{kl} = \1(f(x_k) > f(x_l)) + \frac{1}{2}\1(f(x_k) = f(x_l))\in R^{n\times m}$, $\gamma_{kl} = \frac{W^{(i)}_{kl}}{W^{(j)}_{kl}}$ and $p_{kl} = \frac{W^{(j)}_{kl}}{\sum_{a,b}W^{(j)}_{ab}}$. If we think of $p$ as a probability distribution of $k,l$, then $\lauc(x_i) - \lauc(x_j) = \frac{\text{Cov}_p(\gamma, I)}{\mathbb{E_p}(\gamma)}.$
\end{restatable}

The proof can be found in Supplementary material \ref{proof:AUC-cov}.

\subsubsection{Monotonicity Guarantees}

Let the covariate space $\X$ be partitioned into two disjoint, phenotypically distinct regions $S_1$ and $S_2$. We assume $S_1\cap S_2=\emptyset$ and $S_1\cup S_2=\X$. The weighted AUC score can be regarded as a performance score of a specific model in participants phenotypically similar to \(i\). We claim using a simplified scenario with some assumptions:

\begin{restatable}[Ordering of Localized AUC Expectations Across Model Competence Regions]{proposition}{AUC}
\label{prop:AUC}
If Assumptions 1 and 2 hold, then the population regional expectations of localized AUC functionals satisfy:

$$ \mathbb{E}_{X_i \sim P_{X|S_1}}[\lauc(X_i)] > \mathbb{E}_{X_j \sim P_{X|S_2}}[\lauc(X_j)] $$
\end{restatable}

This justifies why the weighted AUC is a plausible metric for model performance evaluation. The assumptions are in Supplementary material \ref{supp:assumptions}, and the proof is in \ref{proof:AUC}.

To stabilize the variance and address boundary issues\footnote{the estimated AUC is bounded between 0 and 1, and its sampling variance depends on its true value: the variance is largest near 0.5 and shrinks toward 0 as the AUC approaches 0 or 1. This mean–variance dependence can introduce heteroskedasticity and distort the loss function of a regression model. The logit transformation, via the delta method, approximately removes this dependence}, we applied the logit transformation:

\[
z_i = \log\!\left(\frac{\lauc(x_i)}{1 - \lauc(x_i)}\right),
\]

where \(z_i\) was used as the target variable for the subsequent machine learning model.
\subsection{Estimator Construction and Asymptotic Consistency} 
\label{sec:estimator_consitency}

As defined in Definition \ref{def:local_auc_estimator}, we can define the estimator of the localized AUC. 

\begin{restatable}{proposition}{laucustatistic}
If we follow Definitions \ref{def:local_auc} and \ref{def:local_auc_estimator}, then the $\widehat{\lauc}$ estimator is an unbiased consistent estimator of $\lauc$.
\end{restatable}

Details of this can be found in the Supplementary material \ref{proof:lauc-ustatistic}.

\subsection{Asymptotic Normality}

\begin{restatable}[Asymptotic Normality]{proposition}{asymptoticnormality}
\label{prop:asymptoticnormality}
    \begin{align}
        \frac{1}{N}(\widehat{\lauc}(x_i)-\lauc(x_i))\xrightarrow{d}\mathcal{N}\left(0,\frac{\sigma^2}{(\mu_+\mu_-)^2}\right),\quad\text{where}\quad \sigma^2=\frac{\sigma_+^2}{\lambda}+\frac{\sigma_-^2}{1-\lambda}.
    \end{align}
    
\end{restatable}

The proof can be found in the Supplementary material \ref{proof:asymptoticnormality}.

\subsection{Plug-in Standard Errors via Logit Transformation}

\begin{restatable}{proposition}{hypothesistest}
    \label{hypothesistest}
An asymptotically valid $100(1-\alpha)\%$ confidence interval is as follows as $N\to\infty$:
$$\mathbb{P}\left( \lauc \in \left[ \widehat{\lauc}_N - z_{1-\alpha/2} \frac{\widehat{\sigma}_{\lauc}}{\sqrt{N}}, \; \widehat{\lauc}_N + z_{1-\alpha/2} \frac{\widehat{\sigma}_{\lauc}}{\sqrt{N}} \right] \right) \to 1 - \alpha$$
where $z_{1-\alpha/2}$ is the standard normal critical value.
\end{restatable}

The derivation of the confidence interval can be found in \ref{proof:hypothesistest}


\newpage
\section{Proposed LOCATOR Algorithm for Subgroup Identification } 

\label{subsec:methodology}
\subsection{Mathematical Formulation}
\label{subsec:MathFormulation}
Consider the input feature matrix $X\in \R^{n\times P}$ with $X=[x_1,x_2,...,x_n]$ and $x_i\in \R^P$ and an existing binary classification model $f$ (e.g. PREVENT in our case) that is being evaluated by the method. The input data should include all the predictors of the existing model being evaluated and the variables needed to differentiate high-performance from low-performance regions. Denote $S$ as the set of predictor variables of the model $f$; thus, the prediction vector can be computed as $\hat{p}=[f(x_{1,S}),...,f(x_{n,S})]$. 

To quantify phenotypic similarity between individuals, we used Gower distance \cite{gower}, which accommodates mixed continuous and binary variables. Let \(\mathbf{x}_i\) and \(\mathbf{x}_j\) be the baseline feature vectors for participants \(i\) and \(j\). The Gower distance \(d_{ij}\) is defined as the average of per-variable contributions:

\[
d_{ij} = \frac{1}{P} \sum_{p=1}^{P} \delta_{ijp},
\]

where for a continuous variable, \(\delta_{ijp} = \frac{|x_{ip} - x_{jp}|}{\text{range}_p}\), and for a binary variable, \(\delta_{ijp} = \mathbf{1}(x_{ip} \neq x_{jp})\). The distance matrix \(\mathbf{D} = (d_{ij})\) was computed using all baseline phenotypic variables $P$. 

For each index participant \(i\), we define similarity weights:

\[
w_{ij} = \left(1 - d_{ij}\right)^{15},
\]

followed by truncation to ensure non-negative weights, consistent with prior work \cite{oikonomou}. Thus, participants phenotypically closer to \(i\) receive higher weights, forming a local neighborhood around \(i\).

For each index participant \(i\), we evaluate the discrimination performance of the reference model within its local neighborhood by computing it localized AUC - a weighted area under the receiver operating characteristic curve (weighted AUC)\cite{weightedAUC}: 

\[
\lauc(x_i) = \frac{\sum_{k,l: y_k=1, y_l=0} [\mathbf{1}(\hat{p}_k > \hat{p}_l) + \frac{1}{2}\mathbf{1}(\hat{p}_k = \hat{p}_l)] \, w_{ik} w_{il}}{\sum_{k,l: y_k=1, y_l=0} w_{ik} w_{il}},
\]

where \(y\) is the true outcome (CVD death) and \(\hat{p}\) is the predicted risk of the reference model. This produces an individual-focused localized AUC value, \(\lauc(x_i)\), reflecting how well PREVENT discriminates CVD risk in participants phenotypically similar to \(i\).

To predict LAUC we train XGBoost regression model using $\logit(\lauc(x_i))$ as an outcome and using baseline phenotypic features $x_i\in \R^P$.

\subsection{LOCATOR algorithm and implementation details}

\begin{algorithm}[H]
\caption{LOCATOR Algorithm}
\label{alg:proposed_algo}
\begin{algorithmic}[1]
\Require Input data matrix $X \in \R^{n \times d}$, true outcomes y, existing model $f$ predictions \(\hat{p}=f(X)\).
\Ensure Identify  high and low competence regions of risk prediction model \(\hat{p}=f(X)\).
\State Split into train and test dataset: $X_\text{train}, X_\text{test}, y_\text{train}, y_\text{test}\gets \text{split}(X,y)$. Line \(2 - 9\) are training process, and lines \(10 - 11\) are testing process.
\State Compute Gower distances using all baseline variables: $d_{kl} \gets \text{Gower}(x_k, x_l), x_k\in X_{\text{train}}, x_l\in X_{\text{train}}$
\State Compute weights: $w_{kl}\gets (1-d_{kl})^{15}$
\For{Reference participant $i=1$ to $|X_\text{train}|$}
    \State $\lauc(x_i) = \frac{\sum_{k,l: y_k=1, y_l=0} [\mathbf{1}(\hat{p}_k > \hat{p}_l) + \frac{1}{2}\mathbf{1}(\hat{p}_k = \hat{p}_l)] \, w_{ik} w_{il}}{\sum_{k,l: y_k=1, y_l=0} w_{ik} w_{il}}$
$\lauc(x_i)$
\EndFor
\State Train XGBoost model to predict $\logit(\lauc(x_i))$: \(\hat{g}_{\mathrm{XGB}} : X_{\text{train}} \mapsto \operatorname{logit}(\lauc(x_i))\)
\State Select features that are important in the final BGBoost model $ \hat{g}_{\mathrm{XGB}} $
\State Use \(\hat{g}_{\mathrm{XGB}}\) to get predictions in test set: $z_{\text{pred}} \gets \hat{g}_{\mathrm{XGB}}(X_{\text{test}})$
\State Identify low- and high-performance subgroups:
\begin{itemize}
    \item median = median(\(z_{\text{pred}}\))
    \item $X_{\text{low}} \gets z_{\text{pred}}< \text{median}$
    \item $X_{\text{high}} \gets z_{\text{pred}}>= \text{median}$
\end{itemize}
\State Run SHAP or other similar analysis to identify what kind of participants constitute low and high performance subgroups in $ X_{\text{train}}$
\end{algorithmic}
\end{algorithm}

\subsection{Interpretation of High and Low Performance Regions with SHAP analysis}
\label{subsec:data_analysis}
To interpret the XGBoost model $\hat{g}_{\mathrm{XGB}}(X_{\text{test}})$ and identify the baseline characteristics driving local  performance, we propose to employ SHapley Additive exPlanations (SHAP) \cite{SHAP}. For each test sample, SHAP value is computed for every feature, representing the contribution of that feature to the predicted \(z_i\) relative to the average prediction. A SHAP summary plot can be used to identify the order of importance of each variable.  Thus, subgroups can be identified according to SHAP values using the most important variables.

We derived global feature importance as the mean absolute SHAP value per feature. Additionally, we visualized the direction and magnitude of each feature’s effect via summary plots. Features with high importance were selected for identifying heterogeneous performance subgroup, and their effects on predicted local performance were further examined using SHAP dependence plots.

To address potential instability due to overfitting, we also calculated the internal importance of XGBoost 'gain' and compared it with the SHAP results as a sensitivity analysis and removed features with gain weaker than 10\% of the most important feature..

In addition, to verify that the XGBoost model has the ability to accurately predict model performance in a newly given specific sample, we used the predicted $z_i$ in the external test set to separate a high-performance subgroup and a low-performance subgroup by their median. We therefore compared the two ROC curves\cite{ROC} of the two subgroups and calculated the standard AUC value of each subgroup. If the LOCATOR algorithm

Furthermore, we examine the association between predicted local performance and observed model discrimination by computing the classical AUC of the PREVENT model for cumulative subgroups defined by increasing percentiles of the predicted $z_i$ (from the lowest $1\%$ to $100\%$, with step size $1\%$).

\newpage
\section{Simulation studies}
\label{subsec:sim_design}

To further validate our method, we considered several distinct simulation scenarios: two positive examples and one negative control example. 

\subsection{Simulation Scenario 1: Identification of Subgroups with Heterogeneous Model performance (Positive Control)}
First, we assume that there is a predictor variable that separates the whole data set into two distinct  subsets \(D_M\) and \(D_N\). Two binary classification reference models  \(f_M\) and \(f_N\) are such that \(f_i\) has more precise prediction than the other in subset \(S_i\) where \(i\) can be either \(M\) or \(N\).

We assumed \(x_1, x_2, x_3\) i.i.d., where \(x_1, x_2 \sim N(0, 1)\),
and \(x_3 \sim U[0, 1]\) as a latent subgroup identifier.

Two models $f_M$ and $f_N$ are defined as: 
\begin{enumerate}
    \item \(f_M(x_1,x_2) = x_1 - 10 x_2\) 
    \item \(f_N(x_1,x_2) = 10 x_1^2 x_2\)
\end{enumerate}

Split the whole data set into two subgroups: \(D = D_M \cup D_N\), where 

\begin{enumerate}
\def\labelenumi{\arabic{enumi}.}
\tightlist
\item
  For $D_M$ \(x_3<0.5\). For each instance, generate true outcome \(y=1\) with probability \(p_M=\logit^{-1}(f_M(x_1,x_2))\). Otherwise \(y=0\).
\item
  For $D_N$ \(x_3\ge0.5\). For each instance, generate true outcome \(y=1\) with probability \(p_N=\logit^{-1}(f_N(x_1,x_2))\). Otherwise \(y=0\).
\end{enumerate}

Construct a simulated training data set with $5000$ instances and testing data set with $1000$ instances, using the rule above. 

Train two XGBoost models $g_M, g_N$ on training data set, using the LOCATOR algorithm in the whole training dataset.
\begin{itemize}
    \item $g_M$ is trained on logit-transformed localized  $\lauc(x_i)$  of reference model $f_M$.
    \item $g_N$ is trained on logit-transformed localized  $\lauc(x_i)$   of reference model $f_N$.
\end{itemize}

Figure \ref{fig:simu} shows that $g_M$ tends to predict higher $\logit(\lauc(x_i)$ values when \(x_3<0.5\) while $g_N$ tends to predict higher values in the \(x_3\ge0.5\) regime.

Recall from the simulation setup that the entire data set was split into two subsets $D_M, D_N$, where $f_M$ performs better than $f_N$ in data set $D_M$, and vise versa. The split between $D_M, D_N$ si defined by \(x_3<0.5\).

\begin{figure}[htbp]
    \centering
    \fbox{\begin{minipage}{0.7\linewidth}
        \centering
\includegraphics[width=3.37in]{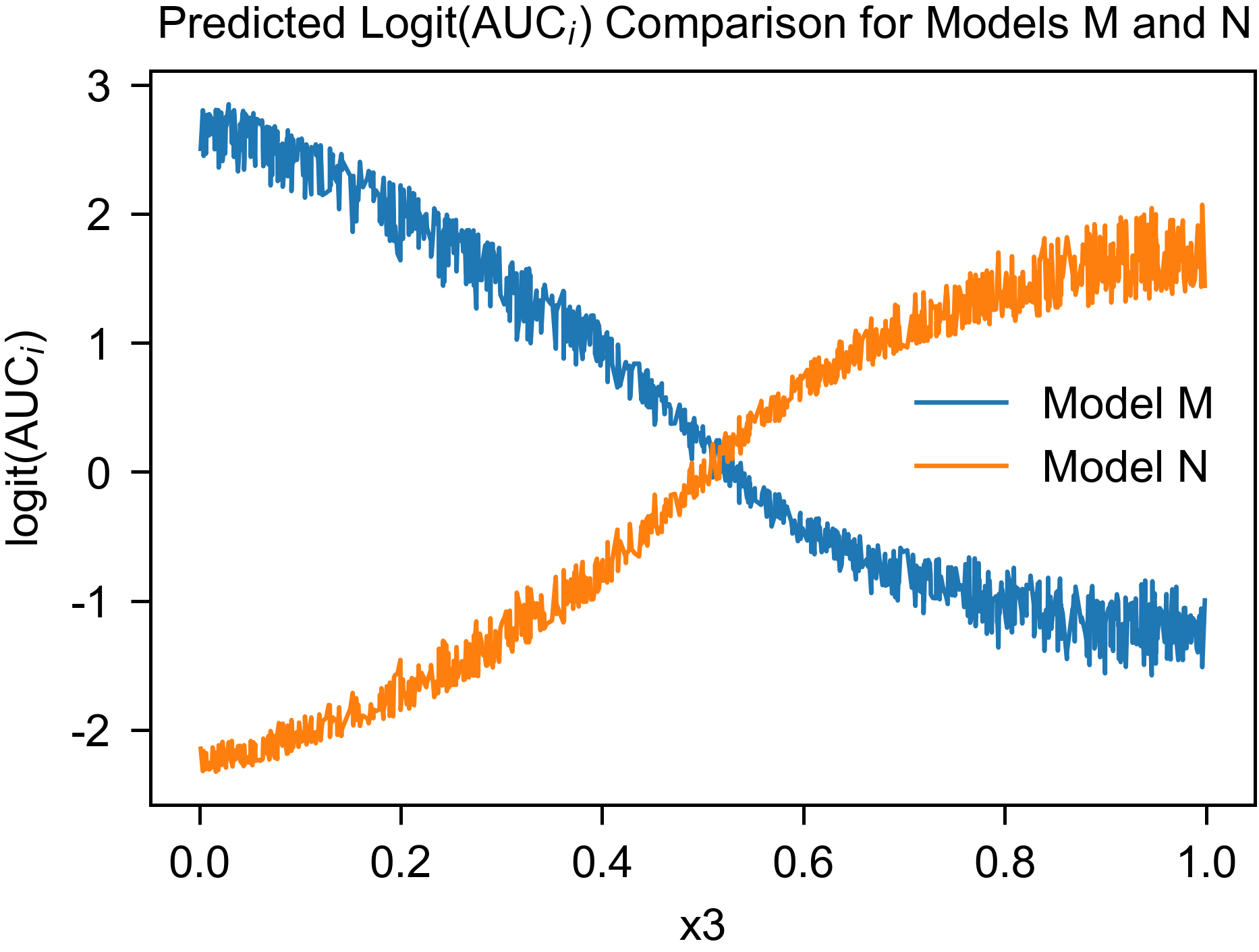}
    \end{minipage}}
    \caption{The comparison of predicted $\logit(\lauc(x_i))$ values between model $g_M$ (\texttt{Model M in legend}) and model $g_N$ (\texttt{Model N in legend}) in a simulated data set. According to the simulation design in section \ref{subsec:sim_design}, model $f_M$ should give more precise prediction when \(x_3<0.5\) than model $f_N$, while model $f_N$ should give more precise prediction when \(x_3\ge 0.5\). $g_M$ and $g_N$ are XGBoost models trained based on the weighted $\logit(\lauc(x_i))$ of $f_M$ and $f_N$ respectively. From the figure it is clear that model $g_M$ has larger $\logit(\lauc(x_i))$ than model $g_N$ when $x_3<0.5$ while model $g_N$ has larger $\logit(\lauc(x_i))$  than model $g_M$ when $x_3\ge 0.5$, which aligns with our expectation.}
    \label{fig:simu}
\end{figure}

This simulations illustrates that localized $\logit(\lauc(x_i))$ is sensitive to $x_3$ that was used to separate \(D_M\) and \(D_N\).

\subsection{Simulation Scenario 2: negative control experiment}
\label{subsec:neg-sim}
To justify that our method works also well when there is no way to differentiate high- and low-performance subgroups, we designed another simulation.

Suppose that there is a dataset which has $6$ variables \(x_1, x_2,\dots,x_6\) and one binary outcome \(y\in \{0,1\}\). The variables follow the probability distribution:
\begin{enumerate}
    \item \(y\) follows Bernoulli distribution \(y\sim\text{Bern}(0.5)\)
    \item \(x_1\) follows conditional normal distribution \(x_1\sim \mathcal{N}(A\cdot y, 1)\) with a chosen parameter \(A=1.812388\)
    \item \(x_2,x_3,\dots,x_6 \overset{\mathrm{iid}}{\sim} \mathcal{N}(0, 1)\)
\end{enumerate}
The reference model being evaluated is the data-generating model \(f(x_1)=\frac{1}{1+e^{-x_1}}\).
So that every baseline variable except \(x_1\) is random noise.

We then applied our method and trained the XGBoost model \(g\) on \(x2,\dots,x_6,y\) (excluding \(x_1\)), so that the model \(g\) is trained purely on the noise. The resulting ROC curves of high\ and low-performance subgroups identified by \(g\) in the holdout test set are shown in Figure \ref{fig:simu-neg}.

\begin{figure}[htbp]
    \centering
    \fbox{\begin{minipage}{0.5\linewidth}
        \centering
    \includegraphics[width=1\linewidth]{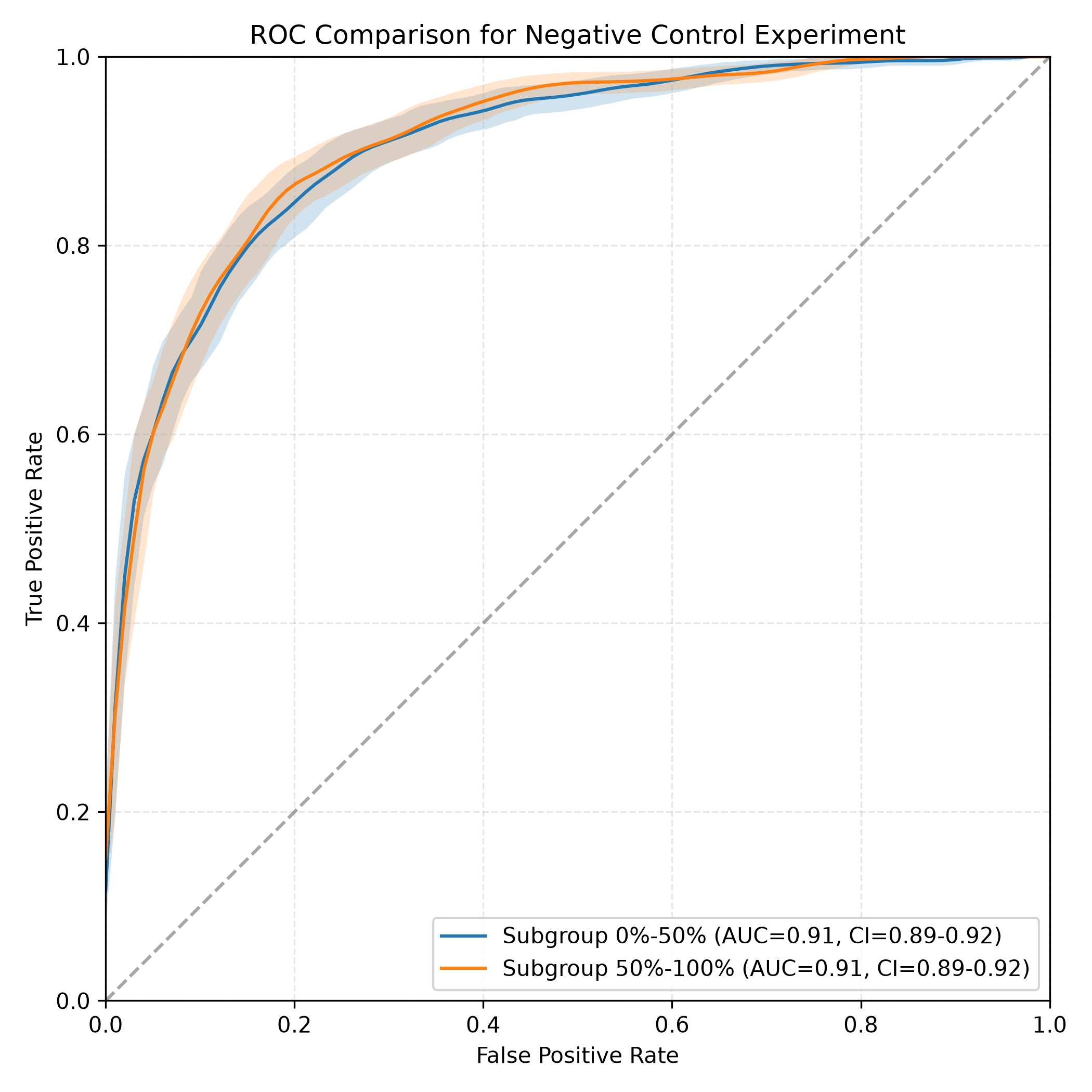}
    \end{minipage}}
    \caption{ROC curves comparison of low- and high-performance subgroups from the negative-control simulation \ref{subsec:neg-sim}. The XGBoost model used to identify the high- and low-performance subgroups is based on pure noise. The confidence bands of the two ROC curves overlap, as expected.}
    \label{fig:simu-neg}
\end{figure}

\subsection{Simulation Scenario 3: Low Performance in the Overlap Region}
\label{subsec:sim-inter}

We set up the following simulation:
\begin{enumerate}
    \item \(y\) follows Bernoulli distribution \(y\sim\text{Bern}(0.5)\)
    \item \(x_1\) follows conditional normal distribution \(x_1\sim \mathcal{N}(A\cdot y, 1)\) with a chosen parameter \(A=1.190232\)
    \item \(x_2 \sim \mathcal{N}(0.17771, 1)\)
    \item \(x_3\sim \mathcal{U}[0,1]\)
\end{enumerate}

This simulations setup is expected to produce unweighted AUC of \(0.8\). The actual unweighted AUC estimated in the test set is \(0.8143\) 

We trained the XGBoost model in the training set and derived the high- and low-performance subgroups in the test set. The ROC curves are also obtained as Figure \ref{fig:simu-inter}. The confidence level is set to be 95\%. The high-performance subgroup has significantly higher ROC curve than low-performance subgroup, and the confidence bands are separable. The standard (unweighted) AUC value for low-performance subgroup is AUC=0.69, with 95\%CI=[0.66; 0.71], which is lower than the standard AUC value for high-performance subgroup  AUC=0.88, with 95\%CI=[0.86-0.90]. 

\begin{figure}[htbp]
    \centering
    \fbox{\begin{minipage}{0.5\linewidth}
        \centering
    \includegraphics[width=1\linewidth]{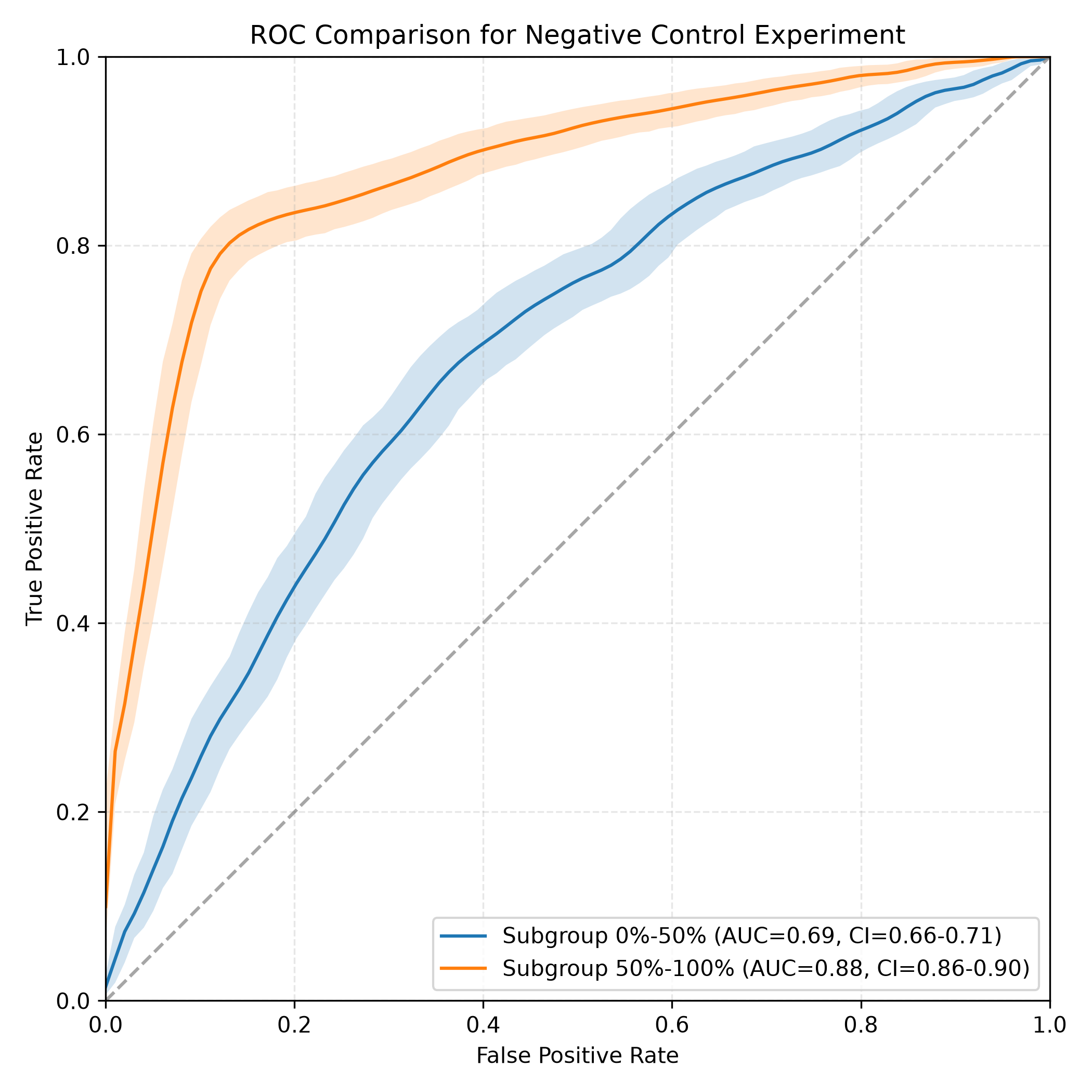}
    \end{minipage}}
    \caption{ROC curves comparison of low- and high-performance subgroups of the simulation as in section \ref{subsec:sim-inter}. The test set is split into low- and high-performance subgroups according to the median of the predictions of the trained XGBoost model. The confidence level is set to be 95\%. The high-performance subgroup has significantly higher ROC curve than low-performance subgroup, and the confidence bands are separable. The standard (unweighted) AUC value for low-performance subgroup is AUC=0.69, with CI=0.66-0.71, which is lower than the standard AUC value for high-performance subgroup  AUC=0.88, with CI=0.86-0.90. The figure gives another validation for our method in identifying performance subgroups of the certain given model.}
    \label{fig:simu-inter}
\end{figure}

Although in this simulation scenario data was produced by the same data-generating model, LOCATOR algorithm still identified high and low performance subgroups. To explain this, we plotted in gray histogram of \(x_1\) in low performance subgroup and in orange in high performance subgroup (Figure \ref{fig:simu-x1}). It is notable that the low-performance subgroup is in the middle of two peaks of the high-performance subgroups. Remember that \(x_1\sim \mathcal{N}(A\cdot y, 1)\) with two normal distribution centers around \(x_1=0\) and \(x_1=A=1.190232\). Therefore, around \(x_1=\frac{A}{2}=0.595116\), two conditional distributions  overlap and it is hard to separate events from non-events in this region. High performance region consists of the left and right boundary of \(x_1\) domain where events and non-events are almost perfectly separated. This example illustrates that a reference model with imperfect discrimination  contains regions with heterogenous performance.

\begin{figure}[htbp]
    \centering
    \fbox{\begin{minipage}{0.5\linewidth}
        \centering
    \includegraphics[width=1\linewidth]{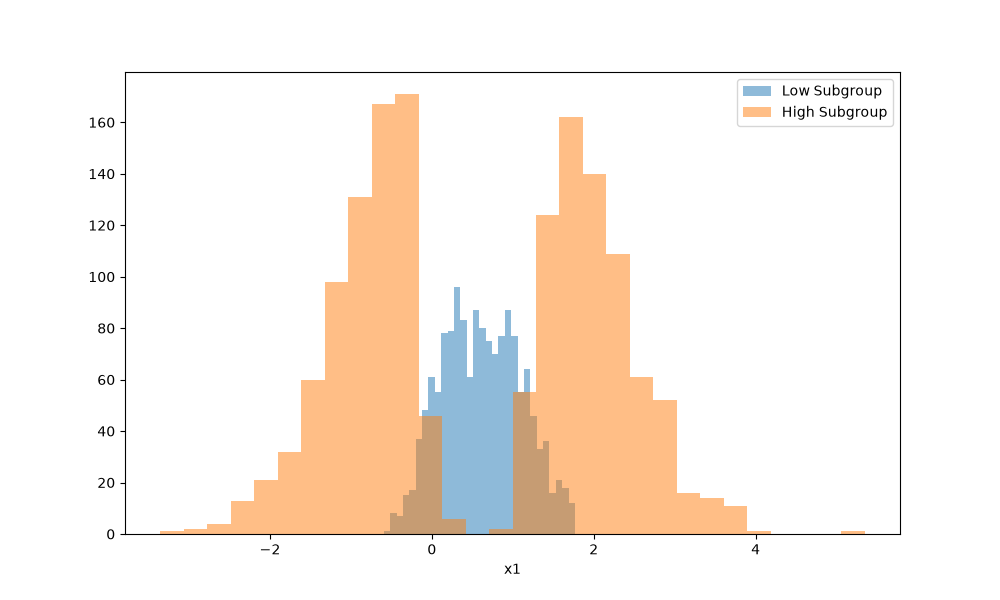}
    \end{minipage}}
    \caption{The distribution of \(x_1\) of low- and high-performance subgroups, identified by the XGBoost model trained in the simulation. The low-performance subgroup is in the middle of two high-performance subgroups, which aligns with the simulation set up \(x_1\sim \mathcal{N}(A\cdot y, 1)\) with two normal distributions overlapping around \(x_1=\frac{A}{2}=0.595116\). So at the middle point it is hard to decide whether the data point comes from the event data \(y=1\) or non-event data \(y=0\).}
    \label{fig:simu-x1}
\end{figure}

\newpage
\section{Disentangling Performance Drivers: Model Heterogeneity vs. Cohort Case-Mix}

Evaluating risk prediction models via localized discrimination functionals-such as the per-person LAUC, may reveal performance gaps between high-competence ($S_1$) and low-competence ($S_2$) subgroups ($\Delta \text{AUC} = \text{AUC}_{S_1} - \text{AUC}_{S_2} > 0$). However, relying solely on a single summary metric like $\Delta \lauc(x_i)$ may conflate fundamentally distinct statistical phenomena . A localized drop in AUC does not automatically imply that the underlying prediction model is structurally flawed. 

To establish actionable guidelines for model remediation versus cohort selection, we formalize a two-axis diagnostic framework that separates \textbf{Model Heterogeneity} (whether the model performs worse than its assigned risks dictate) from \textbf{Cohort Case-Mix / Risk Range Variance} (whether the subgroup naturally lacks outcome-discriminating variance).

\subsection{Local Underperformance Scenarios}

Even when a binary risk model is correctly specified across the entire population, empirical simulations demonstrate that validation cohort characteristics may  generate substantial localized discrimination gaps. We classify localized underperformance into four distinct scenarios based on covariate distribution, predicted risk range, and true model specification.

\subsubsection{Scenario 1.1: Performance Degradation in Overlap Regions}
\textbf{Mechanism:} In Simulation scenario 3 we observed that data-generating model may have low performance region where risk scores overlap and high performance region outside of the ooverlap. This occurs when baseline model discrimination far from imperfect ($\text{AUC} \ll 1.0$), causing substantial overlap between event and non-event risk distributions within specific phenotypic regions.

\textbf{Diagnosis \& Action:} Diagnosed by a wide risk range in low performance region paired low model heterogeneity ($\text{AUC}_{\text{obs}} \ll \text{AUC}_{\text{exp}}$) . This reflects a genuine under-performance of the model. Within low-performance region to discriminate events from non-events better, additional information is required: other biomarker data, imaging etc.

\subsubsection{Scenario 1.2: Homogeneous Covariate Clusters (Validation Cohort Artefact)}
\textbf{Mechanism:} The low-performance subgroup $S_2$ consists of a dense cluster of participants with nearly identical covariate profiles $X$ . Because the model assigns nearly identical predicted risks $\hat{p}_i$ to similar individuals, the localized risk range is compressed .

\textbf{Diagnosis \& Action:} Characterized by a narrow covariate range, narrow predicted risk variance ($\sigma^2_{\text{logit}(\hat{p}) \mid S_2} \approx 0$), and low model heterogeneity ($\text{AUC}_{\text{obs}} \approx \text{AUC}_{\text{exp}}$) . This represents an artifact of sample selection in the validation cohort rather than a defect in the risk model: a model \textit{should} assign identical risks to identical individuals . Refitting the model on the existing covariate set will not improve discrimination . The appropriate remedy is to enrich the validation cohort's covariate diversity or measure novel biomarkers to re-introduce risk variance .

\subsubsection{Scenario 1.3: Covariate Diversity with Flat Risk Range (Biomarker Exhaustion)}
\textbf{Mechanism:} Participants in $S_2$ exhibit diverse covariate vectors $X$, yet the model maps these distinct profiles to a narrow, nearly uniform band of predicted risk . This arises under two distinct conditions:
\begin{enumerate}[label=(\alph*)]
    \item \textit{Omitted Biomarkers:} The existing predictor set lacks the key variables required to stratify risk within this specific sub-population .
    \item \textit{Latent Biomarker Clustering:} An unobserved high-leverage biomarker (with true $\text{AUC} \approx 1.0$) is strongly predicted ($R^2 \approx 1.0$) by observed covariates, but the validation cohort inadvertently enrolled participants sharing an identical value of this latent biomarker .
\end{enumerate}

\textbf{Diagnosis \& Action:} Characterized by reasonable covariate spread alongside a statistically flat risk variance ($R_U < \delta$ via $F$-test) . If driven by omitted predictors, the model must be augmented with novel clinical biomarkers and refit . If driven by latent biomarker clustering, the issue lies in cohort sample selection, implying that ideal validation cohorts should maintain uniform distributions over latent risk factors.

\subsection*{Test for Flat Predicted Risk}

To test whether predicted risk is flat in the low-performance subgroup $S$ compared to its complement $S^c$, we formulate the hypothesis test as:
\begin{equation}
H_0: \frac{\sigma_S^2}{\sigma_{S^c}^2} \ge \delta \quad \text{vs.} \quad H_A: \frac{\sigma_S^2}{\sigma_{S^c}^2} < \delta
\end{equation}
where $S$ is the low-performance subgroup, $S^c$ is its complement, and $\delta \in (0, 1)$ represents the non-flatness threshold ratio.

At the boundary of the null hypothesis ($R = \frac{\sigma_S^2}{\sigma_{S^c}^2} = \delta$), the test statistic scaled by $\delta$ follows an $F$-distribution:
\begin{equation}
\frac{s_S^2}{\delta s_{S^c}^2} \sim F_{n_S-1, \, n_{S^c}-1}
\end{equation}

Because the alternative hypothesis posits a smaller variance in $S$, this is a lower-tail test. The $p$-value is given by:
\begin{equation}
p = P\left(F_{n_S-1, \, n_{S^c}-1} \le \frac{s_S^2}{\delta s_{S^c}^2}\right)
\end{equation}

We reject $H_0$ at significance level $\alpha$ when:
\begin{equation}
\frac{s_S^2}{\delta s_{S^c}^2} \le F_{\alpha; \, n_S-1, \, n_{S^c}-1}
\end{equation}

Equivalently, we can construct a $100(1-\alpha)\%$ one-sided upper confidence bound $R_U$ for the true variance ratio $R = \frac{\sigma_S^2}{\sigma_{S^c}^2}$:
\begin{equation}
R_U = \frac{\hat{R}}{F_{\alpha; \, n_S-1, \, n_{S^c}-1}}
\end{equation}
and reject $H_0$ if $R_U < \delta$.

Derivation of the test is provided in the Appendix.

\subsubsection{Scenario 2: Genuine Local Model Misspecification (Structural Failure)}
\textbf{Mechanism:} The baseline model is structurally misspecified within subgroup $S_2$ due to omitted non-linearities, missing interaction terms, or incorrect variable weights specific to that phenotypic region . Empirical simulations under structural misspecification demonstrate severe discrimination collapse (e.g., $\Delta \text{AUC} \ge +0.35$) .

\textbf{Diagnosis \& Action:} Diagnosed when $S_2$ provides ample risk variance ($\sigma^2_{\text{logit}(\hat{p}) \mid S_2} \gg 0$) yet exhibits a significant excess risk gap ($\text{AUC}_{\text{exp}} - \text{AUC}_{\text{obs}} \gg 0$) . This indicates unambiguous model failure . The remedy requires structural model refitting, localized variable selection, or the introduction of subgroup-specific interaction terms .

\newpage
\section{Application: Localized Discrimination of the PREVENT Model in National Health and Nutrition Examination Survey (NHANES) and External Validation in the UK Biobank}
\label{subsec:real_data}
As a practical example we apply proposed LOCATOR  algorithm to identify high and low performance regions of  PREVENT model. American Heart Association and American College of Cardology currently use PREVENT model to identify patients who would benefit from lipid and blood pressure-lowering treatment. PREVENT model is therefore used for millions of patients. We train LOCATOR algorithm in NHANES data, validate in hold out test subset and externally validate in UK Biobank.
\subsection{NHANES 1999–2010 Cohort Setup and Preprocessing}
 National Health and Nutrition Survey (NHANES) collects data on about $5,000$ adults and children in communities across the United States \cite{cdc_nhanes_about}. We used the cycle range 1999-2010 from NHANES data, with linked National Death Index data through 2019 (N=17,748). We used CVD death as an outcome (N=756).  Data were randomly split into training ($70\%$) and test ($30\%$) subsets stratified by categories of PREVENT predicted risk using the  Python package \emph{scikit-learn}\cite{scikit-learn} . We excluded participants outside of PREVENT predictors' range, with prior CVD,  who are pregnant, and with incomplete set of PREVENT predictor variables. We applied input ranges according to the R package \emph{preventr}\cite{prevetnr}, as shown in Figure \ref{fig:attrition}. Summary statistics of demographics and baseline characteristics in NHANES cycles 1999-2010 are shown in Table \ref{tab:table1} and \ref{tab:table2}.  To address  missing data in variables not used by PREVENT model, we  filtered participants with  more than 40\% of baseline variables missing; we also excluded baseline variables with more than 90\%   missing. After that, we utilized Multiple Imputation by Chained Equations (MICE) with the Python package miceforest to training and testing subsets. 
 The Gower distances, weights, $\lauc(x_i)$  and $z_i=\logit(\lauc(x_i))$ were computed in training data set. We then trained an XGBoost model with input variables including PREVENT predictors and other baseline variables (shown in Table \ref{tab:table1} and \ref{tab:table2}). 
 
Hyperparameters were tuned via random search with 5-fold cross-validation over 100 iterations, exploring ranges for learning rate (0.1–0.3), maximum tree depth (3–10), subsample ratio (0.5–1.0), column subsample ratio (0.5–1.0), and regularization parameters (\(\gamma\), \(\alpha\), \(\lambda\)). Early stopping with 20 rounds of no improvement was applied to prevent overfitting. The final model performance was evaluated on the test set using \(R^2\) and root mean squared error (RMSE). After 5-Fold Cross-Validation hyperparameter tuning, we selected the best parameters.

The model used the tree booster with the objective function znote{The mean squared error (MSE) loss is convex, and XGBoost adds an \(L_2\) regularization term, making the overall objective strongly convex and thus ensuring the convergence of the training loss \cite{Biau2021}. However, due to the greedy algorithms employed for both node splitting and sequential tree addition \cite{XGBoost}\cite{Friedman2001GreedyFA}, the optimization typically converges to a local minimum within the constrained tree-function space rather than the global minimum.}.  To come up with the most parsimonious XGBoost model, we removed predictors that produce smaller  than 10\% gain of the strongest predictor. The  Root-Mean-Square-Error (RMSE) of final XGBoost model was $0.0170$  and the R2 score of $0.9713$ in the test set.

Trained XGBoost model was applied to the test subset which was split into high and low performance subgroup at the median of  predicted  $z_i=\logit(\lauc(x_i))$. Conventional ROC curves were plotted separately in each group and conventional AUC was estimated for each. 

\subsection{Identification of PREVENT Competence Regions in NHANES}

A high and a low performance subgroup were identified by splitting the test set in halves, with $50\%$ percentile of XGBoost prediction $\logit(\lauc(x_i))$. Two ROC curves in the two subgroups were drawn as in Figure \ref{fig:roc}.

The conventional AUC values in high and low performance subgroups are $0.88$ 95\%CI[.82; .93] and $0.78$ 95\%CI[.69; 77] respectively, shown in Table \ref{tab:high-low} and \ref{fig:roc} and \ref{fig:rocUKB}. The two ROC curves shown in Figure \ref{fig:roc} verify direction and the separation of high- and low-performance subgroups identified by our proposed LOCATOR algorithm in holdout test subset (internal validation). The corresponding statistics of size, CVD death number and mean conventional AUC value in each subgroup are shown in Table \ref{tab:high-low}. We plotted conventional AUC in high and low performance subgroups versus percent assigned to low performance subgroup in Figure \ref{fig:auc-ratio}. 

\begin{figure}[htbp]
    \centering
    \fbox{\begin{minipage}{\linewidth}
    \centering
    \includegraphics[width=1\linewidth]{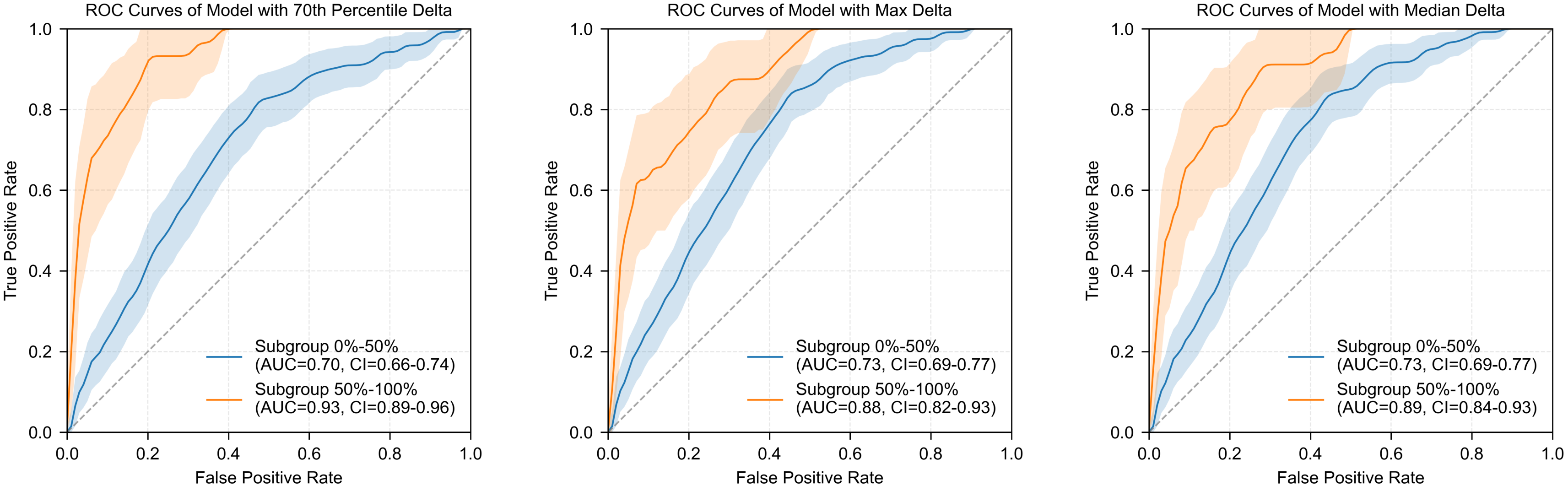}
    \end{minipage}}
    \caption{The three plots show the capability of different XGBoost models with maximum delta (AUC$_i$ difference), 70th percentile delta, and median delta. The two ROC curves in a high- and a low-performance subgroup were identified by splitting above and below the median of XGBoost predictions of $\logit(\lauc(x_i))$. Subgroup \(0\%-50\%\) corresponds to low-performance subgroup and Subgroup \(50\%-100\%\) corresponds to high-performance subgroup. The two ROC curves show  separation, with AUC values shown in the legend significantly different as expected. With confidence level of \(95\%\), the confidence bands are shown in the figure and the confidence intervals of AUC values are shown in the legend. }
    \label{fig:roc}
\end{figure}

A series of conventional AUC values is plotted in Figure \ref{fig:auc-ratio} with respect to the percent in the lowest performance subgroup.

\begin{figure}[htbp]
    \centering
    \fbox{\begin{minipage}{0.5\linewidth}
        \centering
    \includegraphics[width=1\linewidth]{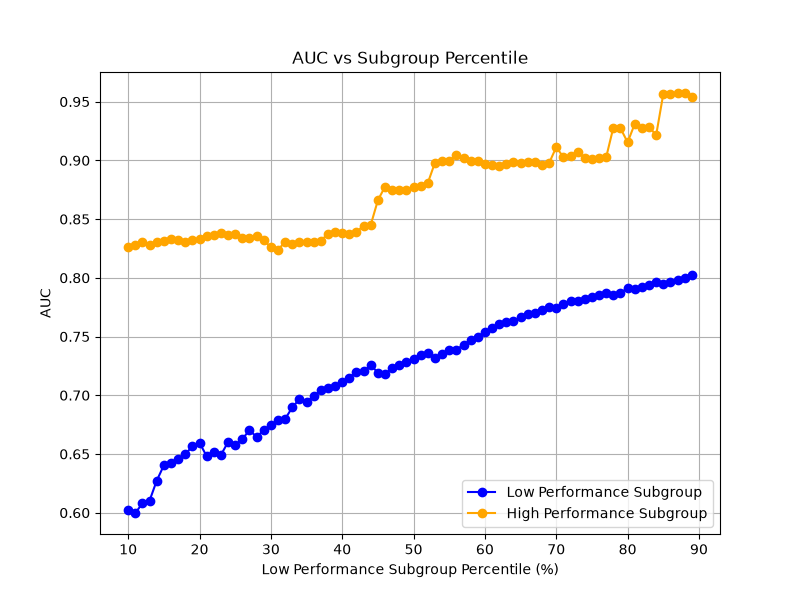}
    \end{minipage}}
    \caption{AUC with respect to the fraction assigned to the  low-and high-performance subgroup in the NHANES data.}
    \label{fig:auc-ratio}
\end{figure}

\begin{table}[]
    \centering
    \begin{tabular}{c|c|c|c}
        Subgroup & Total Number & CVD death number & mean classical AUC value \\
        \hline
        Low perforamance & 2663 & 190 & 0.780 \\
        High performance & 2663 & 37 & 0.880
    \end{tabular}
    \caption{Total sizes, CVD death event numbers and mean AUC values of high and low performance subgroups from test data set}
    \label{tab:high-low}
\end{table}

\subsection{External Validation of PREVENT Competence Regions in UK Biobank}
To demonstrate out-of-sample generalizability of identified competence regions, we applied trained XGBoost model in UK Biobank - a large prospective cohort (N=355,122). Data pre-processing and harmonization of UK Biobank with NHANES data, demographics characteristics and Consort diagram are provided in the Appendix. 

Competence regions learned in the NHANES data were applied to UK Biobank and conventional AUC was calculated in high and low performance subgroups using CVD death as an outcome. AUC in the high-performance subgroup was  0.761 with 95CI 0.742-0.781 and 0.708 with 95CI 0.697-0.719 in the low-performance subgroup in the UK Biobank. These numbers show mild attenuation, but the 95CIs do not overlap and are below the .05 statistical significance threshold. 

\begin{figure}[htbp]
    \centering
    \fbox{\begin{minipage}{\linewidth}
    \centering
    \includegraphics[width=1\linewidth]{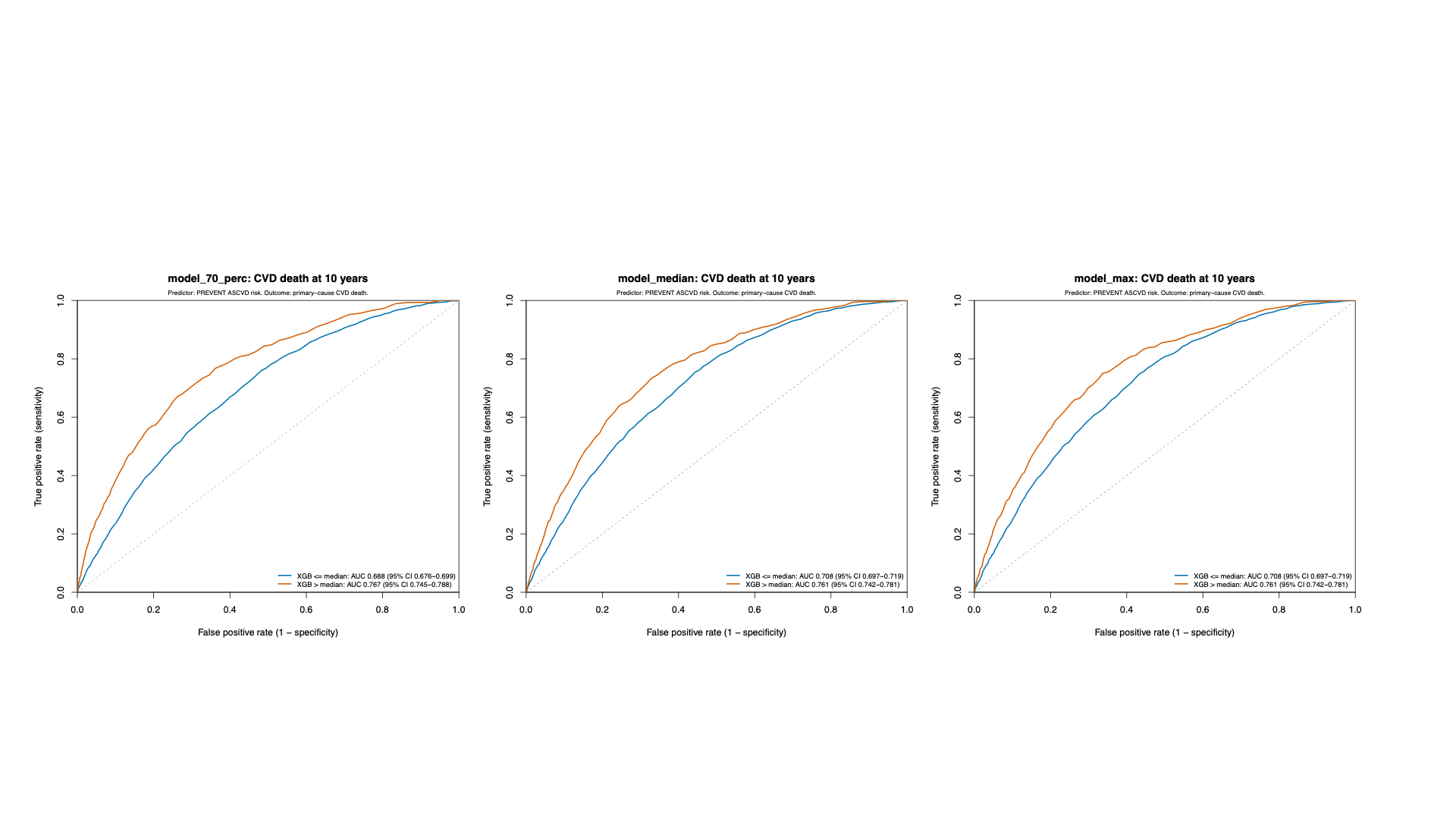}
    \end{minipage}}
    \caption{External Validaton in UK Biobank. The two ROC curves of a high- and a low-performance subgroup were identified by splitting above and below the median of XGBoost predictions of $\logit(\lauc(x_i))$ Subgroup \(0\%-50\%\) corresponds to low-performance subgroup and Subgroup \(50\%-100\%\) corresponds to high-performance subgroup. The two ROC curves show  separation, with AUC values shown in the legend significantly different as expected. With confidence level of \(95\%\), the confidence bands are shown in the figure and the confidence intervals of AUC values are shown in the legend. }
    \label{fig:rocUKB}
\end{figure}

\subsection{Identification of Variables associaited with high and low performance regions with SHAP analysis}
The SHAP summary graph in Figure \ref{fig:shap} shows that the most critical performance drivers are age, followed by sex, taking antihypertensive medications at baseline, current smoking, and the use of statin medication.  SHAP summary plot was obtained from the test dataset, as in Fig \ref{fig:shap}. It shows that older age, blood pressure treatment, being male, smoking, a low poverty income ratio,  taking statin medication, and a history of type 2 diabetes and cancer are associated with low performance of the PREVENT model.

\begin{figure}[htbp]
    \centering
    \fbox{\begin{minipage}{0.5\linewidth}
    \centering
    \includegraphics[width=1\linewidth]{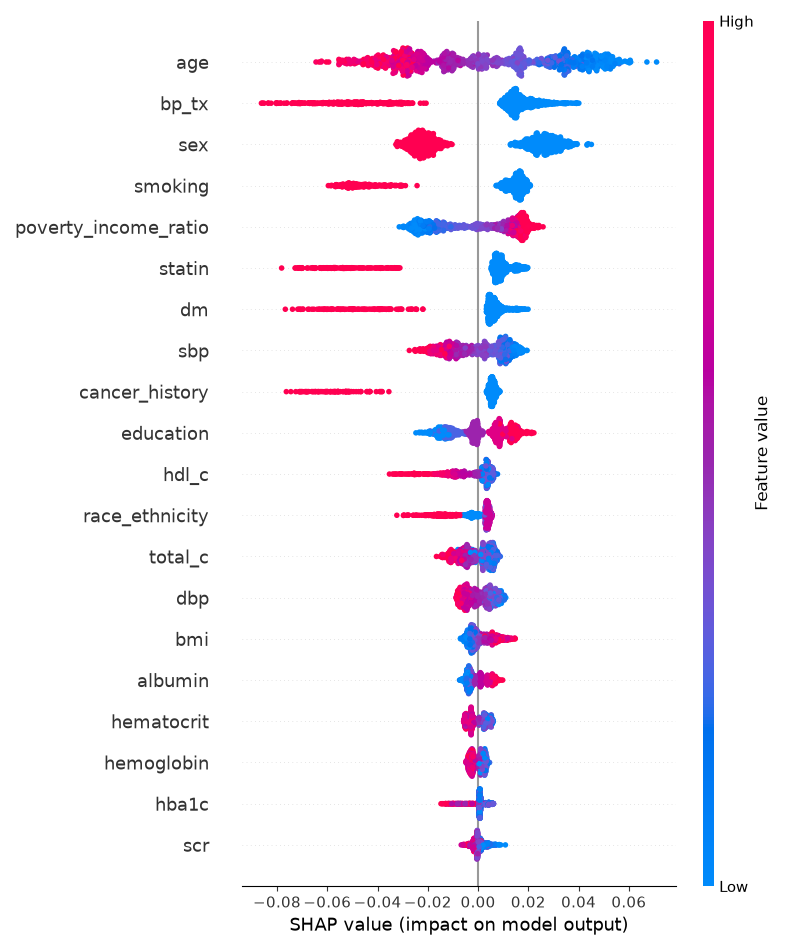}
    \end{minipage}}
    \caption{SHAP summary plot. the most critical variable is age, followed by sex, bp\_tx(now taking prescribed medicine), smoking and statin medication use.}
    \label{fig:shap}
\end{figure}

\section{Discussion}
In this work, we introduced the LOCATOR algorithm and the localized LAUC functional to address the fundamental limitation of global metrics used to validate binary classification models. Current methods of model validation rely on single performance metrics and do not provide information on where in the covariate space the model is underperforming.  LAUC establishes a foundation for individual-level performance assessment and is used in our proposed novel LOCATOR algorithm. The LOCATOR algorithm identifies regions in the covariate space with high or low model performance. By using the localized LAUC functional, LOCATOR avoids data subsetting and loss of statistical power. We provided intuition behind the localized LAUC functional by showing that it is algebraically equal to the ratio of covariance and expectation. We also provided assumptions that guaranty that, on average, the localized LAUC functional  has the right ordering. Using non-parametric two-sample $U$-statistic asymptotic theory, we derived closed-form plug-in variance estimators and logit-stabilized confidence intervals that eliminate the need for computationally intensive resampling in large clinical cohorts. Beyond identifying performance heterogeneity, LOCATOR provides a diagnostic taxonomy to disentangle the causes of underperformance. Through our two-axis diagnostic matrix and a formal $F$-test for flat predicted risk, researchers can distinguish genuine structural model misspecification from validation cohort case-mix restrictions. This distinction may prevent unnecessary model refitting when underperformance is driven by validation cohort homogeneity, while pinpointing specific phenotypic subgroups where structural model updates or novel biomarker discovery are required. Applied to the guidelines-recommended AHA PREVENT model, LOCATOR identified its high and low performance regions in NHANES data ($N = 17,748$; $\text{AUC} = 0.88$ in high-LOCATOR vs. $0.73$ in low-competence regions). The same high and low performance regions were successfully validated in the prospective UK Biobank cohort. SHAP influence decomposition identified older age, antihypertensive treatment, statin therapy, diabetes, and smoking as the main contributors to the lower performance of the PREVENT model. 

\subsection{Remediation Strategies for Low-Competence Regions}

\subsection{Strengths and Limitations}

Several limitations warrant consideration. First, NHANES provides mortality follow-up but lacks non-fatal cardiovascular events, which are part of the PREVENT composite outcome. This discordance may attenuate observed performance and limit comparability with original validation studies. Second, NHANES is cross-sectional, precluding the assessment of dynamic risk factor trajectories that could influence model performance over time. Third, our analysis is restricted to the U.S. population, and generalizability to other settings remains untested. Fourth, the fixed exponential parameter (exp=15) for similarity weighting, while validated in prior work\cite{oikonomou}, may not be optimal for all datasets. Fifth, SHAP associations are not causal; they identify correlates of performance heterogeneity but do not imply causal determinants.

\subsection{Extensions and Future Directions}

Future research should explore adaptive weighting schemes to optimize the exponential parameter for specific datasets, potentially by maximizing subgroup separation or the stability of weighted AUC estimates across bootstrap samples. Extension to other risk models (e.g., PCE, SCORE2) would reveal whether observed heterogeneity is model-specific or systematic. Incorporating time-to-event outcomes and time-dependent AUC metrics would better align with the survival framework of PREVENT. External validation in diverse cohorts-such as UK Biobank or MESA, is essential to assess generalizability and explore racial/ethnic variation. Finally, prospective implementation studies could evaluate the clinical utility of flagging low-performance subgroups to guide targeted model recalibration or supplementary risk assessment.

In addition, a comparison between the division method of XGBoost and by using SHAPley values to divide high- and low-performance subgroups can also be tested, in order to verify the alignment of direct combination of subgroup variable values and XGBoost method and verify the XGBoost method ability in identifying multi-variant interaction and predicting local performance scores. Another comparison between our SHAP summary plot and direct SHAP summary plot derived directly from the existing model (PREVENT) can also be potentially meaningful. 

\subsection{Conclusions}
The novel LOCATOR algorithm and localized LAUC functional proposed in this work can be used to identify low and high performance regions of binary classifiers. As binary classifiers continue to inform high-stakes healthcare decisions for millions of patients, LOCATOR algorithm shifts the model evaluation paradigm from single, pooled cohort summary statistics toward localized performance metrics. We are working on generalizing this approach beyond the clinical domain and risk prediction models.

\newpage
\section{Acknowledgements}

Anna Győrffy-Kerekes would like to thank the ETH Center for Learning Systems for the contributions to her PhD. The authors gratefully acknowledge institutional support from ETH Zurich, Switzerland. This research was conducted using the UK Biobank Resource under application numbers 81959 and 206575. We extend our sincere gratitude to all participants and study organizers of NHANES and UK Biobank for making these datasets available.

\subsection{Statement on AI use}

In the preparation of this work, generative AI tools were used to assist with code writing. Specifically, DeepSeek (\url{https://www.deepseek.com/}) and Gemini (\url{https://gemini.google.com/}) were used to assist with Python code writing. Claude Desktop (free plan, \url{https://claude.ai/download}) was used to assist with R language code writing. These tools were used for code-related assistance, such as suggesting code snippets, debugging, and explaining error messages. All AI-assisted code was reviewed, tested, and adapted by authors, who take full responsibility for the final code and analyses.

During the preparation of this manuscript, large language models, specifically Gemini Pro and Claude Pro were employed to assist with LaTeX debugging, verify the correctness of select proofs, and refine the academic writing. 
The localized AUC functional definition, the LOCATOR algorithm, and Propositions 1-3 with proofs are entirely due to the authors: no use of AI technologies was involved. A conversation with Gemini 3.6 Flash (Google) was used to generate the initial draft of the proofs of Propositions 4-6, which were subsequently reworked by the authors. Text was scanned for errors with the assistance of Gemini 3.6 Flash (Google). Authors take full responsibility for the correctness of the entire content of this manuscript.

\bigskip
\newpage
\begin{center}
{\large\bf SUPPLEMENTARY MATERIAL}
\end{center}



\label{supplement}

\section{Proofs from Section \ref{sec:theory}}
\label{supp:proof_monotonicity}

Let the covariate space $\X$. Let $f:\X\to\R$ be the prediction model. We also define
\begin{align}
    F^+(t)&=\Prob(f(X^+)\leq t),\\
    F^-(t)&=\Prob(f(X^-)\leq t),\\
\end{align}

Let us remind ourselves, that the weighted AUC is defined in Definition \ref{def:local_auc}.

\underuniform*
\begin{proof}
\label{proof:AUCiunderuniform}
\begin{align}
    \lauc(x_i)&=\frac{\EF{w(x_i,X^+)w(x_i,X^-)\idX}}{\EF{w(x_i,X^+)w(x_i,X^-)}}\\
    &=\frac{\EF{c^2\idX}}{\EF{c^2}}\\
    &=\auc.
\end{align}

If we look at the estimator, we get
\begin{align}
    \lauc(x_i)&=\frac{\sum_{k=1}^K\sum_{l=1}^L w(x_i, x_k^+)w(x_i,x_l^-)\idx}{\sum_{k=1}^K\sum_{l=1}^L w(x_i, x_k^+)w(x_i,x_l^-)}\\
    &=\frac{\sum_{k=1}^K\sum_{l=1}^L c^2\idx}{\sum_{k=1}^K\sum_{l=1}^L c^2}\\
    &=\frac{c^2\sum_{k=1}^K\sum_{l=1}^L \idx}{KL c^2}\\
    &= \auc.
\end{align}
\end{proof}

\AUCcov*
\begin{proof}
\label{proof:AUC-cov}

We know
\begin{align}
    \mathbb{E}_p(\gamma)&=\sum_{k,l}p_{k,l}\gamma_{k,l}\\
    &=\sum_{k,l}\left(\frac{W_{kl}^{(i)}}{W_{k,l}^{(j)}}\right)\left(\frac{W_{kl}^{(j)}}{\sum_{a,b}W_{a,b}^{(j)}}\right)\\
    &=\sum_{k,l}\left(\frac{W_{kl}^{(i)}}{\sum_{a,b}W_{a,b}^{(j)}}\right)\\
    &=\frac{\sum_{k,l}W_{kl}^{(i)}}{\sum_{a,b}W_{a,b}^{(j)}}
\end{align}
and
\begin{align}
    \mathbb{E}_p(I)&=\sum_{k,l}p_{k,l} I_{k,l}\\
    &=\sum_{k,l} I_{k,l}\left(\frac{W_{kl}^{(j)}}{\sum_{a,b}W_{a,b}^{(j)}}\right)\\
    &=\lauc(x_j)
\end{align}
and
\begin{align}
    \mathbb{E}_p(\gamma I)&=\sum_{k,l}p_{k,l}\gamma_{k,l} I_{k,l}\\
    &=\sum_{k,l}I_{k,l}\left(\frac{W_{kl}^{(i)}}{W_{k,l}^{(j)}}\right)\left(\frac{W_{kl}^{(j)}}{\sum_{a,b}W_{a,b}^{(j)}}\right)\\
    &=\sum_{k,l}I_{k,l}\left(\frac{W_{kl}^{(i)}}{\sum_{a,b}W_{a,b}^{(j)}}\right)\\
    &=\frac{\sum_{k,l}I_{k,l} W_{kl}^{(i)}}{\sum_{a,b}W_{a,b}^{(j)}}.
\end{align}

Then we have 
\begin{align}
    \frac{\text{Cov}_p(\gamma,I)}{\mathbb{E}_p(\gamma)}&=\frac{\mathbb{E}_p(\gamma I)-\mathbb{E}_p(\gamma)\mathbb{E}_p(I)}{\mathbb{E}_p(\gamma)}\\
    &=\frac{\frac{\sum_{k,l}I_{k,l} W_{kl}^{(i)}}{\sum_{a,b}W_{a,b}^{(j)}}-\lauc(x_j) \frac{\sum_{k,l}W_{kl}^{(i)}}{\sum_{a,b}W_{a,b}^{(j)}}}{\frac{\sum_{k,l}W_{kl}^{(i)}}{\sum_{a,b}W_{a,b}^{(j)}}}\\
    &=\frac{\sum_{k,l}I_{k,l} W_{kl}^{(i)}-\lauc(x_j)\sum_{k,l}W_{kl}^{(i)}}{\sum_{k,l}W_{kl}^{(i)}}\\
    &=\frac{\sum_{k,l}I_{k,l} W_{kl}^{(i)}}{\sum_{k,l}W_{kl}^{(i)}}-\lauc(x_j)\\
    &= \lauc(x_i)-\lauc(x_j).
\end{align}

\end{proof}

Let the covariate space $\X$ be partitioned into two disjoint, phenotypically distinct regions $S_1$ and $S_2$. We assume $S_1\cap S_2=\emptyset$ and $S_1\cup S_2=\X$. Let us also define $(S_1\times S_1)^C$ as $\mathcal{X}\times\mathcal{X}\backslash S_1\times S_1$ and $(S_2\times S_2)^C$ as $\mathcal{X}\times\mathcal{X}\backslash S_2\times S_2.$ We also define
\begin{align}
    F_i^+(t)&=\Prob(f(X^+)\leq t\mid X^+, X^- \in S_i)\quad\text{ for }i\in\{1,2\}\text{ and}\\
    F_i^-(t)&=\Prob(f(X^-)\leq t\mid X^+, X^- \in S_i)\quad\text{ for }i\in\{1,2\}.
\end{align}

\noindent Define the population intra-region and cross-boundary concordance parameters:

\begin{align*}
C_1 &= \mathbb{P}_{F^+ \times F^-}(f(X^+) > f(X^-) \mid X^+, X^- \in S_1) = \int_{-\infty}^{\infty} F_1^-(t) \, dF_1^+(t) \\
C_2 &= \mathbb{P}_{F^+ \times F^-}(f(X^+) > f(X^-) \mid X^+, X^- \in S_2) = \int_{-\infty}^{\infty} F_2^-(t) \, dF_2^+(t) \\
C_{1,2} &= \mathbb{P}_{F^+ \times F^-}(f(X^+) > f(X^-) \mid X^+ \in S_1, X^- \in S_2) = \int_{-\infty}^{\infty} F_2^-(t) \, dF_1^+(t) \\
C_{2,1} &= \mathbb{P}_{F^+ \times F^-}(f(X^+) > f(X^-) \mid X^+ \in S_2, X^- \in S_1) = \int_{-\infty}^{\infty} F_1^-(t) \, dF_2^+(t) \\
C_1^{\text{cross}} &= \mathbb{P}_{F^+ \times F^-}(f(X^+) > f(X^-) \mid (X^+, X^-) \notin S_1 \times S_1) \\
C_2^{\text{cross}} &= \mathbb{P}_{F^+ \times F^-}(f(X^+) > f(X^-) \mid (X^+, X^-) \notin S_2 \times S_2)
\end{align*}

\subsection{Assumptions}
\label{supp:assumptions}

\begin{assumptions}[First-Order Stochastic Dominance]

The reference model risk score distributions satisfy strict First-Order Stochastic Dominance (FSD) between regions $S_1$ and $S_2$:

\begin{itemize}
    \item \textbf{Cases ($Y = 1$):}\\ $f(X_1^+) \underset{st}{\ge} f(X_2^+) \iff F_1^+(t) \le F_2^+(t) \quad \forall t \in \mathbb{R}, X_1^+ \sim F^+(\cdot \mid S_1), X_2^+ \sim F^+(\cdot \mid S_2)$
    
    \item \textbf{Controls ($Y = 0$):}\\ $f(X_1^-) \underset{st}{\le} f(X_2^-) \iff F_1^-(t) \ge F_2^-(t) \quad \forall t \in \mathbb{R}, X_1^- \sim F^-(\cdot \mid S_1), X_2^- \sim F^-(\cdot \mid S_2)$
\end{itemize}
\end{assumptions}

\begin{assumptions}[The groups are seperated with high probability]

Let $W_i=\EF{w(X_i, X^+)w(X_i,X^-)}$ and  $W_j=\EF{w(X_j, X^+)w(X_j,X^-)}$

For $D_1, D_2,\epsilon, \eta\in(0,1)$, we have
\begin{align}
   &\ProbF{{\Eij{i}{1}{\frac{w(X_i,X^+)w(X_i,X^-)}{W_i}}}>D_1\lvert(X^+, X^-)\in (S_1\times S_1)}>\epsilon\quad\text{and,}\\
   &\ProbF{{\Eij{j}{2}{\frac{w(X_j,X^+)w(X_j,X^-)}{W_j}}<D_2}\lvert (X^+, X^-)\in (S_2\times S_2)^C}>\epsilon\quad\text{and,}\\
   &\Prob((X^+, X^-)\in S_1\times S_1)>\frac{1}{4}-\eta\quad\text{and,}\\
   &\Prob((X^+, X^-)\in (S_2\times S_2)^C)<\frac{1}{4}+\eta\quad\text{and,}\\
   &\Prob((X^+, X^-)\in S_2\times S_2)<\frac{1}{4}+\eta\quad\text{and,}\\
   &(\epsilon-(1-C_1))D_1\left(\frac{1}{4}-\eta\right)\geq \left(C_2+(C_2^{cross}(1-\epsilon)+C_2^{cross}D_2\epsilon\right)\left(\frac{1}{4}+\eta\right).
\end{align}
\end{assumptions}

\subsection{Useful Lemma}

\begin{lemma}
    \begin{align}
        C_2\leq C_1^{cross}\leq C_1\quad \text{and}\quad C_2\leq C_2^{cross}\leq C_1.
    \end{align}
\end{lemma}
\begin{proof}
    \begin{align}
        C_1^{cross} &=\mathbb{P}(f(X^+) > f(X^-) \mid X^+, X^- \notin S_1 \times S_1)\\
        &= \mathbb{P}(f(X^+) > f(X^-) \mid X^+, X^- \in (S_1 \times S_2) \cup (S_2 \times S_1) \cup (S_2 \times S_2)\\
        &= \frac{\mathbb{P}(f(X^+) > f(X^-), X^+, X^- \in (S_1 \times S_2) \cup (S_2 \times S_1) \cup (S_2 \times S_2)}{\Prob(X^+, X^- \in (S_1 \times S_2) \cup (S_2 \times S_1) \cup (S_2 \times S_2))}\\
        &= \frac{\mathbb{P}(f(X^+) > f(X^-), X^+, X^- \in (S_1 \times S_2))}{\Prob(X^+, X^- \in (S_1 \times S_2) \cup (S_2 \times S_1) \cup (S_2 \times S_2))}\\
        &+\frac{ \mathbb{P}(f(X^+) > f(X^-), X^+, X^- \in (S_2 \times S_1))}{\Prob(X^+, X^- \in (S_1 \times S_2) \cup (S_2 \times S_1) \cup (S_2 \times S_2))}\\
        &+\frac{ \mathbb{P}(f(X^+) > f(X^-), X^+, X^- \in (S_2 \times S_2))}{\Prob(X^+, X^- \in (S_1 \times S_2) \cup (S_2 \times S_1) \cup (S_2 \times S_2))}\\
        &= \frac{\mathbb{P}(f(X^+) > f(X^-)\mid X^+, X^- \in (S_1 \times S_2))\Prob(X^+, X^- \in (S_1 \times S_2))}{\Prob(X^+, X^- \in (S_1 \times S_2) \cup (S_2 \times S_1) \cup (S_2 \times S_2))}\\
        &+\frac{ \mathbb{P}(f(X^+) > f(X^-)\mid X^+, X^- \in (S_2 \times S_1))\Prob(X^+, X^- \in (S_2 \times S_1))}{\Prob(X^+, X^- \in (S_1 \times S_2) \cup (S_2 \times S_1) \cup (S_2 \times S_2))}\\
        &+\frac{ \mathbb{P}(f(X^+) > f(X^-)\mid X^+, X^- \in (S_2 \times S_2))\Prob(X^+, X^- \in (S_2 \times S_2))}{\Prob(X^+, X^- \in (S_1 \times S_2) \cup (S_2 \times S_1) \cup (S_2 \times S_2))}\\
        &=\frac{C_{1,2} \Prob(X^+, X^- \in (S_1 \times S_2))+C_{2,1}\Prob(X^+, X^- \in (S_2\times S_1))+ C_2\Prob(X^+, X^- \in (S_2 \times S_2))}{ \Prob(X^+, X^- \in (S_1 \times S_2))+\Prob(X^+, X^- \in (S_2\times S_1))+ \Prob(X^+, X^- \in (S_2 \times S_2))}\\
        &= \lambda_{12}C_{1,2} + \lambda_{21}C_{2,1} + \lambda_{22}C_2\text{, a convex combination of Cs.}\\
    \end{align}
    where $\lambda_{1,2}=\Prob(X^+, X^- \in (S_1 \times S_2)), \lambda_{2,1}\Prob(X^+, X^- \in (S_2\times S_1))$ and $\lambda_2= \Prob(X^+, X^- \in (S_2 \times S_2))$ and $\lambda_{1,2},\lambda_{2,1},\lambda_2\in[0,1]$.\\
    \noindent Similarly, 
    \begin{align}
        C_2^{cross} &= \mathbb{P}(f(X^+) > f(X^-) \mid X^+, X^- \notin S_2 \times S_2\\
        &= \mathbb{P}(f(X^+) > f(X^-) \mid X^+, X^- \in (S_1 \times S_2) \cup (S_2 \times S_1) \cup (S_1 \times S_1)\\
        &= \gamma_{12}C_{1,2} + \gamma_{21}C_{2,1} + \gamma_{11}C_1\text{, a convex combination of Cs.}
    \end{align}
    So it is enough to show, $C_2\leq C_{1,2}, C_{2,1}\leq C_1$.\\
    \textbf{1. Lower Bound:} $C_{1,2} \ge C_2$.

    \noindent By Assumption 1: $f(X_1^+) \underset{st}{\ge} f(X_2^+) \iff F_1^+(t) \le F_2^+(t), \forall t$ (cases in $S_1$ stochastically dominate cases in $S_2$). Because $F_2^-(t)$ is a non-decreasing function and $F_1^+(t) \le F_2^+(t)$:
    
    \begin{align*}
    C_{1,2} &= \int_{-\infty}^{\infty} F_2^-(t) \, dF_1^+(t) = [F_2^-(t)F_1^+(t)]_{-\infty}^{\infty} - \int_{-\infty}^{\infty} F_1^+(t) \, dF_2^-(t) \\
    &= 1 - \int_{-\infty}^{\infty} F_1^+(t) \, dF_2^-(t) \ge 1 - \int_{-\infty}^{\infty} F_2^+(t) \, dF_2^-(t) = C_2
    \end{align*}
    
    \noindent \textbf{2. Upper Bound:} $C_{1,2} \le C_1$.
    
    \noindent By Assumption 1: $F_2^-(t) \ge F_1^-(t), \forall t$ (controls in $S_1$ have stochastically lower scores than controls in $S_2$):
    
    $$C_{1,2} = \int_{-\infty}^{\infty} F_2^-(t) \, dF_1^+(t) \le \int_{-\infty}^{\infty} F_1^-(t) \, dF_1^+(t) = C_1$$
    
    \noindent Thus, $C_2 \le C_{1,2} \le C_1$. Simlarly, we can show $C_2\leq C_{2,1}\leq C_1$.
\end{proof}

\subsection{Main Theorem}

\AUC*

\begin{proof}
\label{proof:AUC}
    \begin{align}
    \Eij{i}{1}{\lauc(X_i)}&=\Eij{i}{1}{\frac{\EF{w(X^+,X_i)w(X^-,X_i)\idX}}{W_i}}\\
    &=\Eij{i}{1}{\EF{\frac{w(X^+,X_i)w(X^-,X_i)\idX}{W_i}}}\\
    &=\EF{\Eij{i}{1}{\frac{w(X^+,X_i)w(X^-,X_i)\idX}{W_i}}}\\
    &=\EF{\idX\Eij{i}{1}{\frac{w(X^+,X_i)w(X^-,X_i)}{W_i}}}.
    \end{align}
    Similarly, 
    \begin{align}
        \Eij{j}{1}{\lauc(X_j)}&=\EF{\idX\Eij{j}{1}{\frac{w(X^+,X_j)w(X^-,X_j)}{W_j}}}.
    \end{align}
    Let $D_1(X_i)=\Eij{i}{1}{\frac{w(X^+,X_i)w(X^-,X_i)}{W_i}}$. Now by intersection our expectations with $S_1\times S_1$ and $S_2\times S_2$ we find:
    \begin{align}
    &\Eij{i}{1}{\lauc(X_i)}=\\
    &=\EF{\idX D_1(X_i)\lvert (X^+, X^-)\in S_1\times S_1}\Prob\left((X^+,X^-)\in S_1\times S_1\right)+\\
    &+\EF{\idX D_1(X_i)\lvert ((X^+, X^-)\in (S_1\times S_1)^C}\Prob\left((X^+,X^-)\in (S_1\times S_1)^C\right).
    \end{align}
    Similarly, if 
    $D_2(X_j)=\Eij{j}{2}{\frac{w(X^+,X_j)w(X^-,X_j)}{W_j}}$, then we have
     \begin{align}
    &\Eij{j}{2}{\lauc(X_j)}=\\
    &=\EF{\idX D_2(X_j)\lvert (X^+, X^-)\in S_2\times S_2}\Prob\left((X^+,X^-)\in S_2\times S_2\right)+\\
    &+\EF{\idX D_2(X_j)\lvert ((X^+, X^-)\in (S_2\times S_2)^C}\Prob\left((X^+,X^-)\in (S_2\times S_2)^C\right).
    \end{align}
    By Assumption 2 \ref{supp:assumptions} we have
    \begin{align}
    &\EF{\idX D_1(X_i)\lvert (X^+, X^-)\in S_1\times S_1}\Prob\left((X^+,X^-)\in S_1\times S_1\right)\geq\\
    &\geq(\epsilon-(1-C_1))D_1\left(\frac{1}{4}-\eta\right).
    \end{align}
    \begin{align}
        \EF{\idX D_1(X_i)\lvert ((X^+, X^-)\in (S_1\times S_1)^C}\Prob\left((X^+,X^-)\in (S_1\times S_1)^C\right)\geq 0.
    \end{align}
    Hence, we have
    \begin{align}
        \Eij{i}{1}{\lauc(X_i)}\geq (\epsilon-(1-C_1))D_1\left(\frac{1}{4}-\eta\right).
    \end{align}
    Similarly, we have
    \begin{align}
        &\EF{\idX D_2(X_j)\lvert (X^+, X^-)\in S_2\times S_2}\Prob\left((X^+,X^-)\in S_2\times S_2\right)\leq\\
        &\leq C_2\left(\frac{1}{4}+\eta\right).
    \end{align}
    and
    \begin{align}
        &\EF{\idX D_2(X_j)\lvert ((X^+, X^-)\in (S_2\times S_2)^C}\Prob\left((X^+,X^-)\in (S_2\times S_2)^C\right)\leq\\
        &\leq (C_2^{cross}(1-\epsilon)+C_2^{cross}D_2\epsilon)\left(\frac{1}{4}+\eta\right).
    \end{align}
    Hence, we have
    \begin{align}
        \Eij{j}{2}{\lauc(X_j)}\leq  C_2\left(\frac{1}{4}+\eta\right)+(C_2^{cross}(1-\epsilon)+C_2^{cross}D_2\epsilon)\left(\frac{1}{4}+\eta\right).
    \end{align}
    So if $(\epsilon-(1-C_1))D_1\left(\frac{1}{4}-\eta\right)\geq \left(C_2+(C_2^{cross}(1-\epsilon)+C_2^{cross}D_2\epsilon\right)\left(\frac{1}{4}+\eta\right)$, then we are finished, which is part of Assumption 2.
    
\end{proof}

\begin{remark}
    Note that for example 
    \begin{align}
        D_1&=1-\frac{\delta\gamma}{4(2C_1\delta+\gamma)}\\
        D_2&=\frac{\delta\gamma}{2(\gamma-1)(2\gamma-\delta)}\\
        \epsilon &=1-\frac{\delta}{2\gamma}\\
        \eta&=\frac{\delta\gamma}{4(\gamma(16C_1-3\delta)-8\delta}\\
        \delta&=C_1-C_2, \gamma=1+C_2^{cross}
    \end{align}
    would make
    \begin{align}
    (\epsilon-(1-C_1))D_1\left(\frac{1}{4}-\eta\right)\geq \left(C_2+(C_2^{cross}(1-\epsilon)+C_2^{cross}D_2\epsilon\right)\left(\frac{1}{4}+\eta\right)
    \end{align}
    true.
\end{remark}

\subsection{Estimator Construction and Asymptotic Consistency}

    Let us define the following kernel:
    \begin{align}
        \psi(x, y)=w(x_i, x)w(x_i,y)\mathbbm{1}(f(x)>f(y))
    \end{align}

    Let us also say,
    \begin{align}
        U_{K,L}&=\frac{1}{KL}\sum_{k=1}^K\sum_{l=1}^Lw(x_i, X^+)w(x_i, X^-)\idX\\
        &=\frac{1}{KL}\sum_{k=1}^K\sum_{l=1}^L\psi(x_k, x_l)\quad\text{and,}\\
        W_+&=\frac{1}{K}\sum_{k=1}^Kw(x_i, x_k^+)\quad\text{and,}\\
        W_-&=\frac{1}{L}\sum_{l=1}^Lw(x_i, x_l^-).
    \end{align}

    Let us assume that $N=K+L\to\infty$ and $\frac{K}{N}\to\lambda\in(0,1)$. Then let 
    \begin{align}
        \mu_U&=\EF{w(x_i,X^+)w(x_i, X^-)\idX}=\EF{\psi(X^+, X^-)}\\
        \mu_+&=\EF{w(x_i,X^+)}\\
        \mu_-&=\EF{w(x_i,X^-)}.
    \end{align}
    With this notation $\lauc(x_i) =\frac{\mu_U}{\mu_+\mu_-}$ and $\widehat{\lauc}(x_i)=\frac{U_{KL}}{W_+W_-}$.

    \textbf{Regularity Assumptions:}
    \label{ass:regularity}
    \begin{itemize}
        \item \textbf{(A1)} Finite, non-zero mean: $0<\mu_+, \mu_-<\infty$
        \item \textbf{(A2)} Finite, non-zero variance: $0<\text{Var}(X^+),\text{Var}<\infty.$
        \item \textbf{(A3)} Non-degeneracy: 
        $\text{Var}_{X^+\sim F^+}(\mathbb{E}_{X^-\sim F^-}(\psi(X^+,X^-))),\text{Var}_{X^-\sim F^-}(\mathbb{E}_{X^+\sim F^+}(\psi(X^+,X^-))) > 0$.
    \end{itemize}

\laucustatistic*

\begin{proof}
    \label{proof:lauc-ustatistic}
    We will rely on two-sample U-statistics theory to derive asymptotic properties of $\widehat{\lauc}$. Let us start by looking at our estimator:
    \begin{align}
        \widehat{\lauc}(x_i)&=\frac{\sum_{k=1}^K\sum_{l=1}^L w(x_i, x_k^+)w(x_i,x_l^-)\idx}{\sum_{k=1}^K\sum_{l=1}^L w(x_i, x_k^+)w(x_i,x_l^-)}\\
        &=\frac{\sum_{k=1}^K\sum_{l=1}^L w(x_i, x_k^+)w(x_i,x_l^-)\idx}{\left(\sum_{k=1}^Kw(x_i, x_k^+)\right)\left(\sum_{l=1}^L w(x_i,x_l^-)\right)}
    \end{align}

    By the Weak Law of Large Numbers (WLLN) for sample averages:
    \begin{align}
        W_+ &\xrightarrow{p} \mu_{+}\quad\text{and,}\\
        W_- &\xrightarrow{p} \mu_{-}.
    \end{align}
    By the Law of Large Numbers for Two-Sample $U$-Statistics, $U_{K,L} \xrightarrow{p} \mu_U$.\\
    By Continuous Mapping Theorem on function $g(u, v, w) = \frac{u}{v w}$ (continuous at $\mu_{^+}, \mu_{^-} > 0$), $\frac{U_{K,L}}{W_+W_-} \xrightarrow{p} \frac{\mu_U}{\mu_{+} \mu_{-}} =\lauc$.
        
\end{proof}

\asymptoticnormality*
\begin{proof}
    \label{proof:asymptoticnormality}
    By asymptotic normality of generalized two sample U-statistic \cite{lee1990ustatistics}, we have
    \begin{align}
        \sqrt{N}(U_{K,L}-\mu_U)\xrightarrow{d}\mathcal{N}(0,\sigma^2), \quad\text{where}\quad \sigma^2=\frac{\sigma_+^2}{\lambda} +\frac{\sigma_-^2}{1-\lambda}.
    \end{align}
    In this equation $\lambda = \lim_{N\to\infty} K/N$ and
    \begin{align}
        \sigma_+^2&=\text{Var}_{X^+\sim F^+}\left(\mathbb{E}_{X^-\sim F^-}\left[\psi(X^+,X^-)\right]\right)\\
        \sigma_-^2&=\text{Var}_{X^-\sim F^-}\left(\mathbb{E}_{X^+\sim F^+}\left[\psi(X^+,X^-)\right]\right).
     \end{align}
     By Slutzky \cite{Slutsky1925-ud}, we have 
     \begin{align}
         \sqrt{N}(\widehat{\lauc}(x_i)-\lauc(x_i))&=\sqrt{N}\left(\frac{U_{K,L}}{W_+W_-}-\frac{\mu_U}{\mu_+\mu_-}\right)\\
         &=\sqrt{N}\left(\frac{U_{K,L}-\mu_U}{W_+W_-}+\left(\frac{\mu_U}{W^+W^-}-\frac{\mu_U}{\mu_+\mu_-}\right)\right)\\
         &=\left(\frac{\sqrt{N}(U_{K,L}-\mu_U)}{W_+W_-}+\left(\frac{\mu_U}{W_+W_-}-\frac{\mu_U}{\mu_+\mu_-}\right)\right)\\
         &=\mathcal{N}(0, \sigma^2/(\mu_+\mu_-)^2).
     \end{align}
\end{proof}

\hypothesistest*

\begin{proof}
\label{proof:hypothesistest}

We aim to build an etimator $\sigma^2/(\mu_+\mu_-)^2$, and on this path we need estimators of $\sigma_+/(\mu_+\mu_-)^2$ and $\sigma_-/(\mu_+\mu_-)^2$. Let's call the estimator of $\sigma^2/(\mu_+\mu_-)^2$ as $\hat{\sigma}_{\lauc}$ and the estimator of $\sigma_+/(\mu_+\mu_-)^2$ and $\sigma_-/(\mu_+\mu_-)^2$ $\hat{\sigma}_{+,\lauc}^2$ and  $\hat{\sigma}_{-,\lauc}^2$ respectively. An estimator of $\mathbb{E}_{X^-\sim F^-}\left[\psi(X^+,X^-)\right]$, we have
\begin{align}
    \psi^+(x^+_k)=w(x_i,x^+_k)\left(\frac{1}{L}\sum_{l=1}^Lw(x_i,x_l^-)\idx \right),
\end{align}
and similarly an estimator of $\mathbb{E}_{X^+\sim F^+}\left[\psi(X^+,X^-)\right]$, we have
\begin{align}
    \psi^-(x^-_l)=w(x_i,x^-_l)\left(\frac{1}{K}\sum_{k=1}^Kw(x_i,x_k^+)\idx\right).
\end{align}

Let us also define:
\begin{align}
    \hat{\mu}_+=\frac{1}{K}\sum_{k=1}^Kw(x_i,x_k^+)\quad\text{and}\quad
    \hat{\mu}_-=\frac{1}{L}\sum_{l=1}^Kw(x_i,x_l^-).
\end{align}
Hence consistent estimators of the variances are
\begin{align}
   \hat{\sigma}_{+,\lauc}^2=\frac{1}{K-1}\sum_{k=1}^K\left(\frac{\psi^+(x^+_k)}{\hat{\mu}_+\hat{\mu}_-}-\widehat{\lauc}_N\right)^2\quad&\text{and}\quad
    \hat{\sigma}_{-,\lauc}^2=\frac{1}{L-1}\sum_{l=1}^L\left(\frac{\psi^-(x^-_l)}{\hat{\mu}_+\hat{\mu}_-}-\widehat{\lauc}_N\right)^2\\
    \hat{\sigma}^2_{\lauc}=&\frac{N}{K}\hat{\sigma}_{+,\lauc}+\frac{N}{L}\hat{\sigma}_{-,\lauc}
\end{align}

Finally the asymptotically valid $100(1-\alpha)\%$ confidence interval is as follows as $N\to\infty$:
$$\mathbb{P}\left( \lauc \in \left[ \widehat{\lauc}_N - z_{1-\alpha/2} \frac{\widehat{\sigma}_{\lauc}}{\sqrt{N}}, \; \widehat{\lauc}_N + z_{1-\alpha/2} \frac{\widehat{\sigma}_{\lauc}}{\sqrt{N}} \right] \right) \to 1 - \alpha$$
where $z_{1-\alpha/2}$ is the standard normal critical value.
\end{proof}

\subsection{Test of homogeneous predicted risk}
To test whether predicted risk is flat in the low-performance subgroup $S$ compared to its complement $S^c$, we formulate the hypothesis test as:
\begin{equation}
H_0: \frac{\sigma_S^2}{\sigma_{S^c}^2} \ge \delta \quad \text{vs.} \quad H_A: \frac{\sigma_S^2}{\sigma_{S^c}^2} < \delta
\end{equation}
where $S$ is the low-performance subgroup, $S^c$ is its complement, and $\delta \in (0, 1)$ represents the non-flatness threshold ratio.

Since $S$ and $S^c$ do not share common observations, we can use a one-sided $F$-test to test $H_0$. Let $n_S$ and $n_{S^c}$ be the sample sizes of the respective regions, and define the sample variances of logit-transformed predicted risks as:
\begin{equation}
s_S^2 = \text{Var}\left(\text{logit}(\hat{p}_i) : i \in S\right), \quad s_{S^c}^2 = \text{Var}\left(\text{logit}(\hat{p}_i) : i \in S^c\right)
\end{equation}

Define the observed variance ratio:
\begin{equation}
\hat{R} = \frac{s_S^2}{s_{S^c}^2}
\end{equation}

At the boundary of the null hypothesis ($R = \frac{\sigma_S^2}{\sigma_{S^c}^2} = \delta$), the test statistic scaled by $\delta$ follows an $F$-distribution:
\begin{equation}
\frac{s_S^2}{\delta s_{S^c}^2} \sim F_{n_S-1, \, n_{S^c}-1}
\end{equation}

Because the alternative hypothesis posits a smaller variance in $S$, this is a lower-tail test. The $p$-value is given by:
\begin{equation}
p = P\left(F_{n_S-1, \, n_{S^c}-1} \le \frac{s_S^2}{\delta s_{S^c}^2}\right)
\end{equation}

We reject $H_0$ at significance level $\alpha$ when:
\begin{equation}
\frac{s_S^2}{\delta s_{S^c}^2} \le F_{\alpha; \, n_S-1, \, n_{S^c}-1}
\end{equation}

Equivalently, we can construct a $100(1-\alpha)\%$ one-sided upper confidence bound $R_U$ for the true variance ratio $R = \frac{\sigma_S^2}{\sigma_{S^c}^2}$:
\begin{equation}
R_U = \frac{\hat{R}}{F_{\alpha; \, n_S-1, \, n_{S^c}-1}}
\end{equation}
and reject $H_0$ if $R_U < \delta$.

\subsection*{Assumptions}
\begin{enumerate}[leftmargin=*]
    \item \textbf{Normality within each region:} $\text{logit}(p) \mid S \sim \mathcal{N}(\mu_S, \sigma_S^2)$ and $\text{logit}(p) \mid S^c \sim \mathcal{N}(\mu_{S^c}, \sigma_{S^c}^2)$.
    \item \textbf{Independence of observations:} Predictions must be independent across subjects. This holds under standard i.i.d. sampling, but clustered or longitudinal settings require block resampling adjustments.
    \item \textbf{Independence of the two variance estimates:} Because $S$ and $S^c$ are non-overlapping sets of independent observations, $s_S^2$ and $s_{S^c}^2$ are statistically independent. Comparing $S$ with its complement $S^c$ is therefore cleaner than comparing $S$ with the entire dataset.
    \item \textbf{Hold-out validation required for $S$:} The region $S$ must be fixed independently of the data used for the test. If an algorithm searches the same sample to discover $S$ and evaluate the $F$-test, selection bias invalidates the nominal $F$-distribution. This test must be conducted on an independent hold-out dataset.
\end{enumerate}


\subsection{Supplementary Figures and Tables}
\label{subsec:figures_tables}
Figure \ref{fig:attrition} shows the attrition statistics of the NHANES data 1999-2010. We initially identified $62,160$ respondents from the NHANES 1999–2010 cycles (A–F). We restricted the sample to adults aged 30–79 years (the age range for the PREVENT equations), leaving $23,920$ individuals. After excluding pregnant women ($n=466$) and individuals with self-reported prior cardiovascular disease ($n=2,272$), $21,182$ remained. We further excluded participants with missing data on any of the 11 PREVENT predictor variables (n=$3,151$), resulting in $18,031$ individuals with complete data. Of these, $259$ were excluded because the preventr scoring algorithm flagged out‑of‑range input values, leaving $17,772$ valid risk scores. Finally, we restricted to participants eligible for mortality linkage, yielding the final analytical cohort of $17,748$ individuals.

\begin{figure}[htbp]
\centering
\fbox{\begin{minipage}{0.2\textwidth}
\centering
\includegraphics[width=1\linewidth]{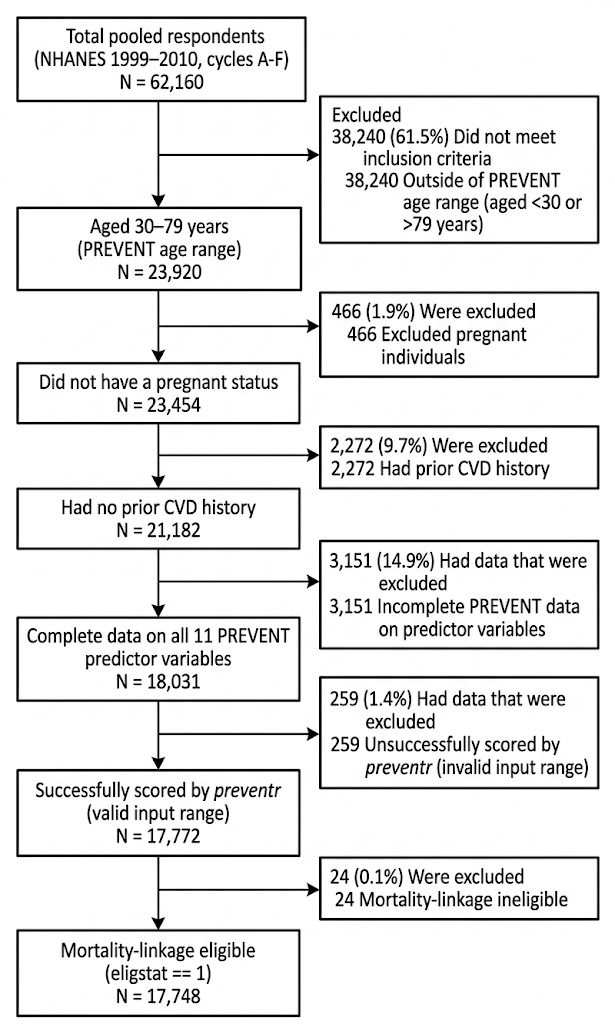}
\end{minipage}}
\caption{Attrition of NHANES data. Initially  there are $62,160$ respondents from the NHANES 1999–2010 cycles (A–F). We restricted the sample to adults aged 30–79 years (the age range for the PREVENT equations), leaving $23,920$ individuals. After excluding pregnant women ($n=466$) and individuals with self-reported prior cardiovascular disease ($n=2,272$), $21,182$ remained. We further excluded participants with missing data on any of the 11 PREVENT predictor variables ($n=3,151$), resulting in $18,031$ individuals with complete data. Of these, $259$ were excluded because the preventr scoring algorithm flagged out‑of‑range input values, leaving $17,772$ valid risk scores. Finally, we restricted to participants eligible for mortality linkage, yielding the final analytical cohort of $17,748$ individuals.}
\label{fig:attrition}
\end{figure}

A statistic of the NHANES data before prepocessing is shown in Table \ref{tab:table1} and \ref{tab:table2}.

UMAP analysis did not reveal obvious clusters in the data.
Before the implementation of our main method, we utilized a dimensionality projection method UMAP\cite{umap} to create more intuitive figures of the input data with respect to several baseline variables (age, sex, etc.). 

After preprocessing (filter out participants with more than $40\%$ variables missing and variables with more than $90\%$ missing data), the distribution of several baseline variables were presented via UMAP visualization method. Figure \ref{fig:umap-train} shows the distribution of age, sex (male=1, female=0), education (Less Than 9th Grade=1, 9-11th Grade (Includes 12th grade with no diploma)=2, High School Grad/GED or Equivalent=3, Some College or AA degree=4, College Graduate or above=5, Refused=7, Don't Know=9), race/ethnicity(Mexican American=1, Other Hispanic=2, Non-Hispanic White=3, Non-Hispanic Black=4, Other Race - Including Multi-Racial=5) and the ratio of poverty to income in training data set. Figure \ref{fig:umap-test1,fig:umap-test2,fig:umap-test3} shows the distribution of same set of variables in testing data set.

\begin{figure}[htbp]
    \centering
        \fbox{\begin{minipage}{0.3\linewidth}
        \centering
        \includegraphics[width=1\linewidth]{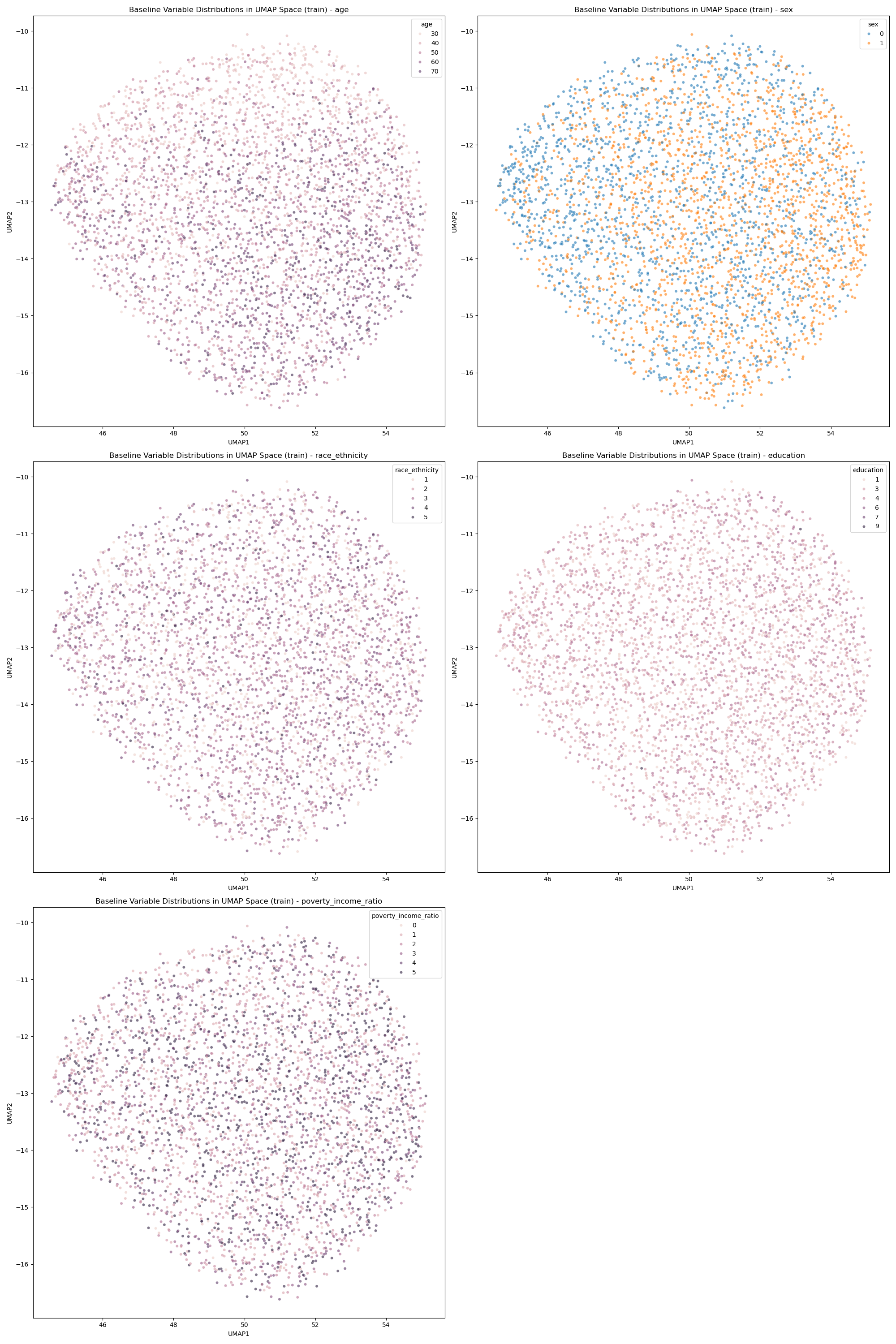}
        \end{minipage}}
    \caption{Absence of UMAP clusters in NHANES datset. UMAP visualization of the distribution of age, sex (male=1, female=0), education (Less Than 9th Grade=1, 9-11th Grade (Includes 12th grade with no diploma)=2, High School Grad/GED or Equivalent=3, Some College or AA degree=4, College Graduate or above=5, Refused=7, Don't Know=9), race/ethnicity(Mexican American=1, Other Hispanic=2, Non-Hispanic White=3, Non-Hispanic Black=4, Other Race - Including Multi-Racial=5) and the ratio of poverty to income in training data set. From the figure it can be seen that there isn't obvious cluster of subgroups.}
    \label{fig:umap-train}
\end{figure}

\begin{figure}
    \centering
    \fbox{\begin{minipage}{0.7\linewidth}
    \centering
    \includegraphics[width=1\linewidth]{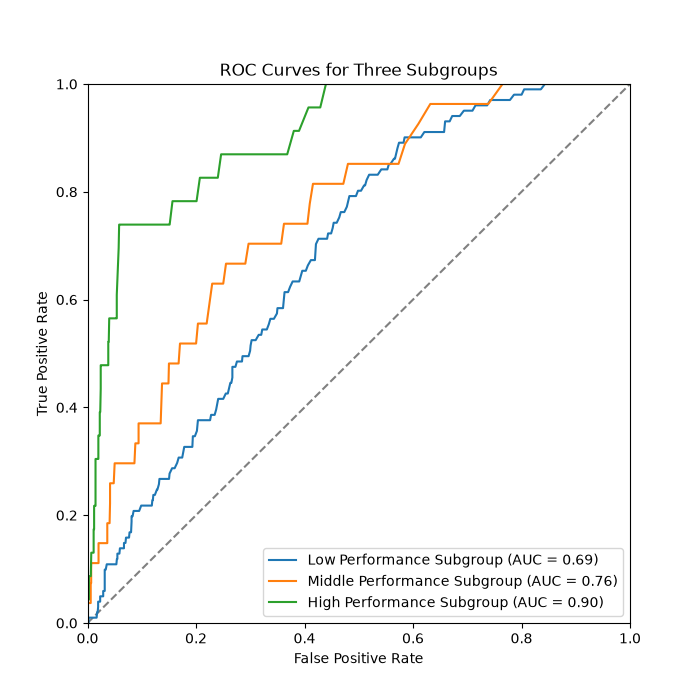}
    \end{minipage}}
    \caption{Three ROC curves of three models split using \(33\%\) and \(67\%\) percentiles of XGBoost predictions $\logit(\lauc(x_i)
    )$. Subgroup \(0\%-33\%\) corresponds to low-performance subgroup, Subgroup \(33\%-67\%\) corresponds to medium-performance subgroup, and Subgroup \(67\%-100\%\) corresponds to high-performance subgroup. The ROC curves of the three subgroups shows increasing performance when climbing from low- to high-performance subgroup. The AUC values shown in the legend also validate this. With confidence level of \(95\%\), the confidence band is shown in the figure and the confidence intervals of AUC values are shown in the legend.}
    \label{fig:roc(h-m-l)}
\end{figure}
\subsection{\textbf{Inclusion and Exclusion Criteria into External Validation Cohort.}
}
Participants aged 30 to 79 years at the time of baseline assessment (instance 0) in the UK Biobank were included, consistent with PREVENT criteria. Participants were excluded if pregnant or if at baseline have reported a history of cardiovascular disease (CVD), including coronary heart disease, myocardial infarction, or stroke. Participants with missing XGBoost input variables were excluded too. All necessary clinical, laboratory, and demographic predictors in the UK Biobank were harmonized to align with the coding conventions, measurement units, and truncation thresholds used in the NHANES training dataset. Harmonization decisions between NHANES variables and UK Biobank are summarized in Figure \ref{fig:mapping_table_NHANES_UKB}.

\begin{figure}
    \centering
    \fbox{\begin{minipage}{1\linewidth}
    \centering
    \includegraphics[width=1\linewidth]{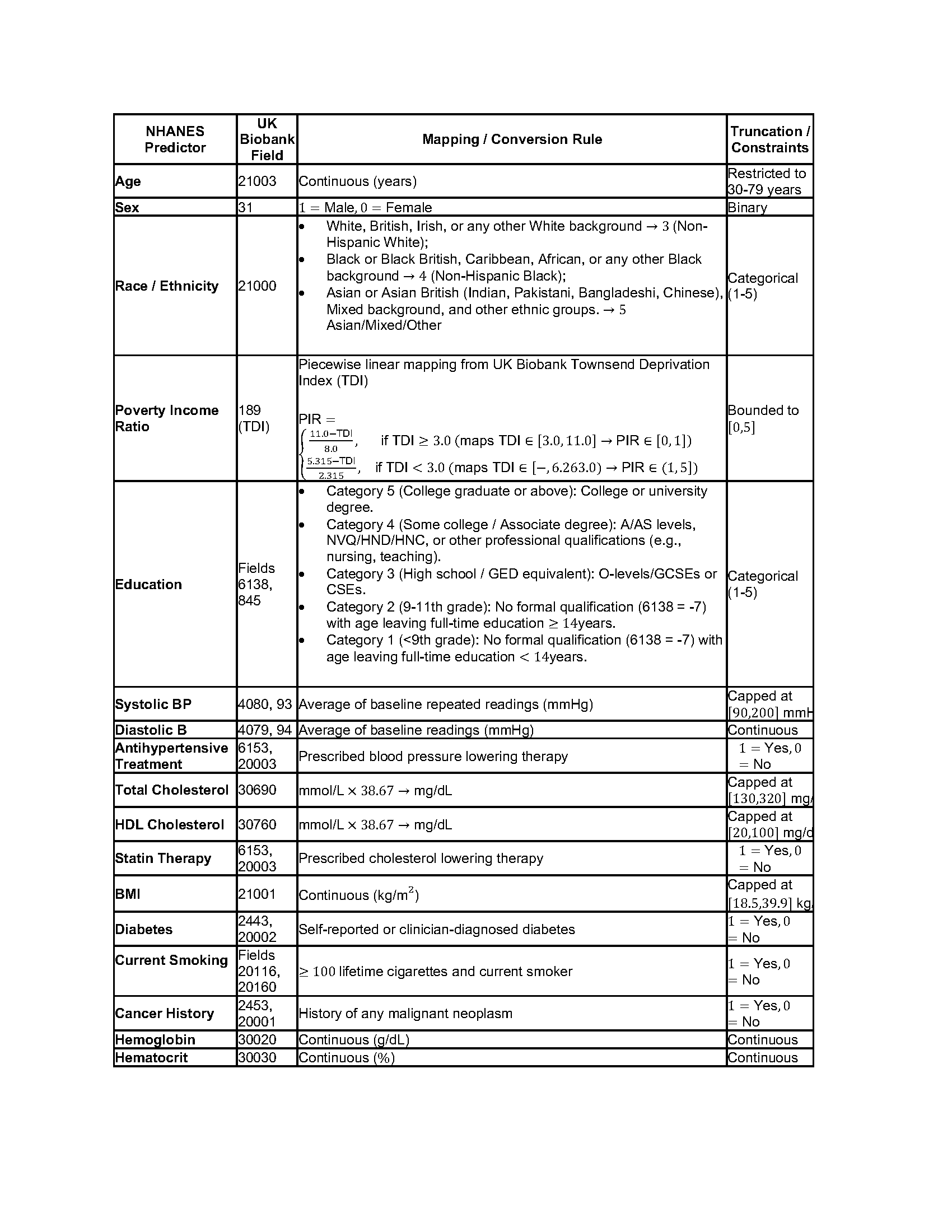}
    \end{minipage}}
    \caption{Harmonization summary of NHANES and UK Biobank variable definitions}
    \label{fig:mapping_table_NHANES_UKB}
        \end{figure}

\begin{figure}
    \centering
    \fbox{\begin{minipage}{1\linewidth}
    \centering
    \includegraphics[width=1\linewidth]{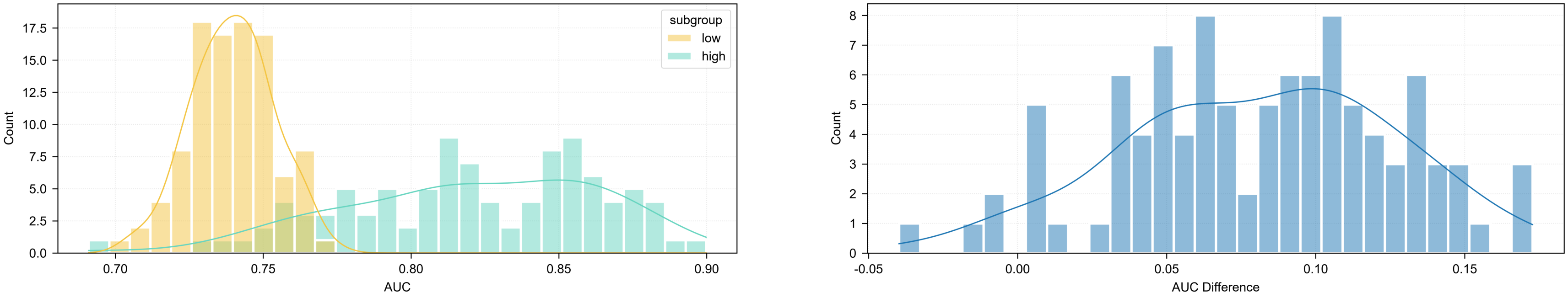}
    \end{minipage}}
    \caption{The histogram of standard AUC values of low- and high-performance subgroups in the test set of NHANES data with 100 seeds (0-99). The whole method is implemented: train the XGBoost model in the training set, then identify low- and high-performance subgroups in the test set. Then calculate the unweighted AUC in low- and high-performance subgroups respectively. Repeat the procedure for 100 seeds from 0 to 99 and plot the histogram using all AUC values. The figure shows that low-performance subgroup has lower AUC than high-performance subgroup. The mean AUC of low-performance subgroup is 0.744, with standard variation 0.0157. The mean AUC of high-performance subgroup is 0.822, with standard vatiation 0.0394.}
    \label{fig:auc_overlap}
\end{figure}

\subsection{Cardiovascular Endpoint Definition}

We validated our model using CVD death since this is the outcome which was used in the NHANES dataset to train LOCATOR model. 

\begin{figure}[htbp]
    \centering
    \fbox{\begin{minipage}{0.3\linewidth}
        \centering
    \includegraphics[width=1\linewidth]{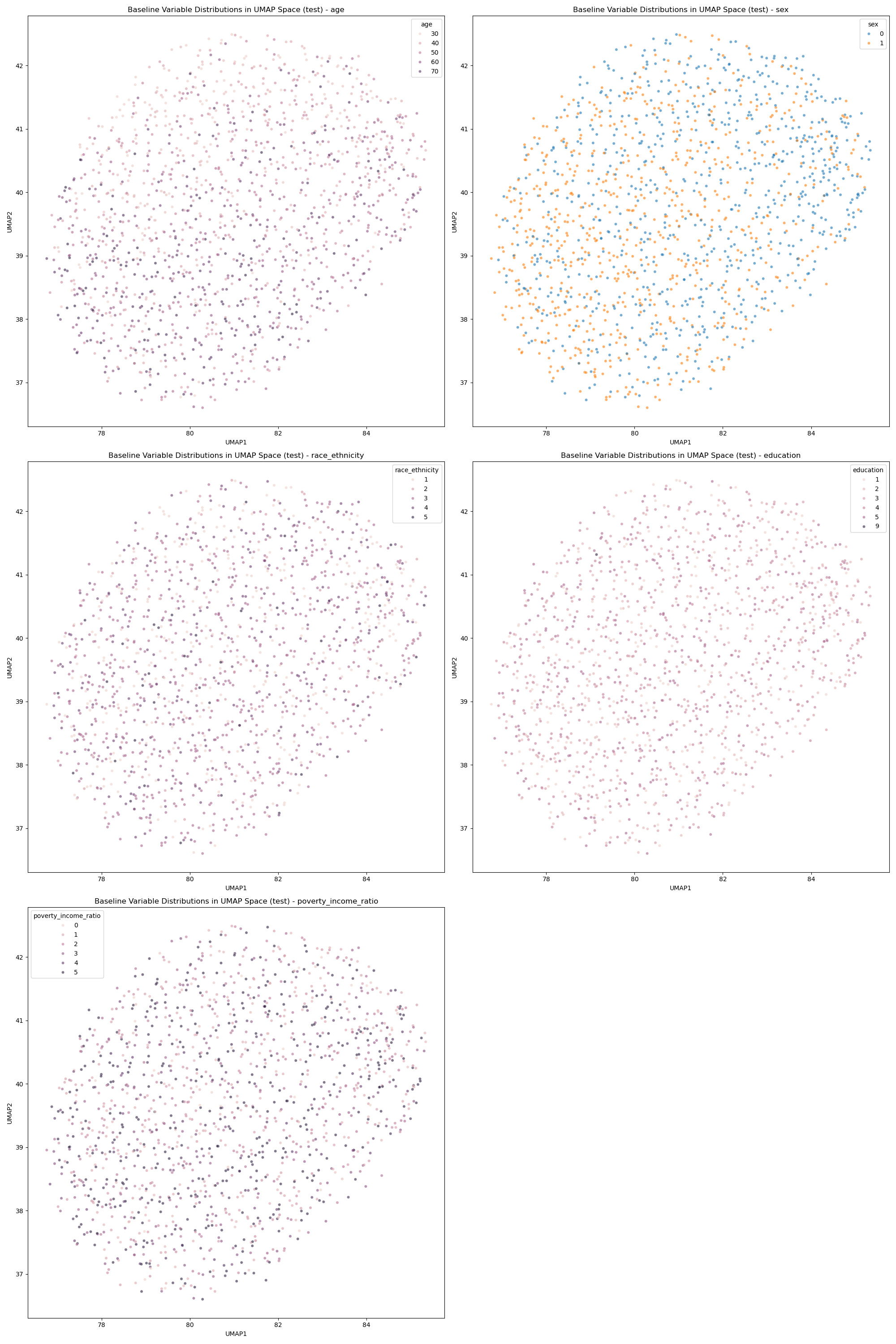}
    \end{minipage}}
    \caption{UMAP visualization of the distribution of age, sex (male=1, female=0), education (Less Than 9th Grade=1, 9-11th Grade (Includes 12th grade with no diploma)=2, High School Grad/GED or Equivalent=3, Some College or AA degree=4, College Graduate or above=5, Refused=7, Don't Know=9), race/ethnicity(Mexican American=1, Other Hispanic=2, Non-Hispanic White=3, Non-Hispanic Black=4, Other Race - Including Multi-Racial=5) and the ratio of poverty to income in \textbf{testing} data set.}
    \label{fig:umap-test1}
\end{figure}

\begin{figure}[htbp]
    \centering
    \fbox{\begin{minipage}{0.3\linewidth}
        \centering
    \includegraphics[width=1\linewidth]{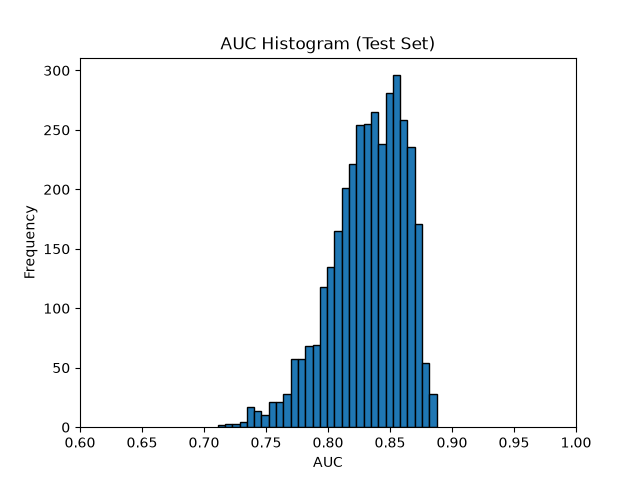}
    \end{minipage}}
    \caption{The histogram of the calculated weighted AUC in test set. The chosen seed is 64, corresponding to the maximum AUC-delta parsimonious model. 
Mean AUC: 0.8260, Median AUC: 0.8282}
    \label{fig:umap-test2}
\end{figure}

\begin{figure}[htbp]
    \centering
    \fbox{\begin{minipage}{0.3\linewidth}
        \centering
    \includegraphics[width=1\linewidth]{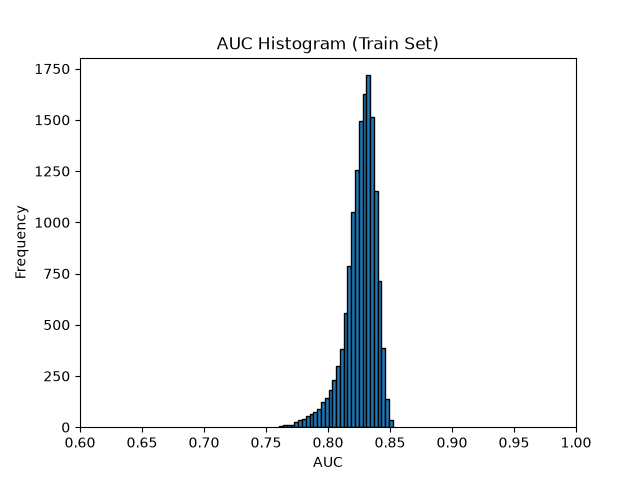}
    \end{minipage}}
    \caption{The histogram of the calculated weighted AUC in train set. The chosen seed is 64, corresponding to the maximum AUC-delta parsimonious model. 
Mean AUC: 0.8321, Median AUC: 0.8359}
    \label{fig:umap-test3}
\end{figure}

\begin{figure}[htbp]
    \centering
    \fbox{\begin{minipage}{0.5\linewidth}
        \centering
    \includegraphics[width=1\linewidth]{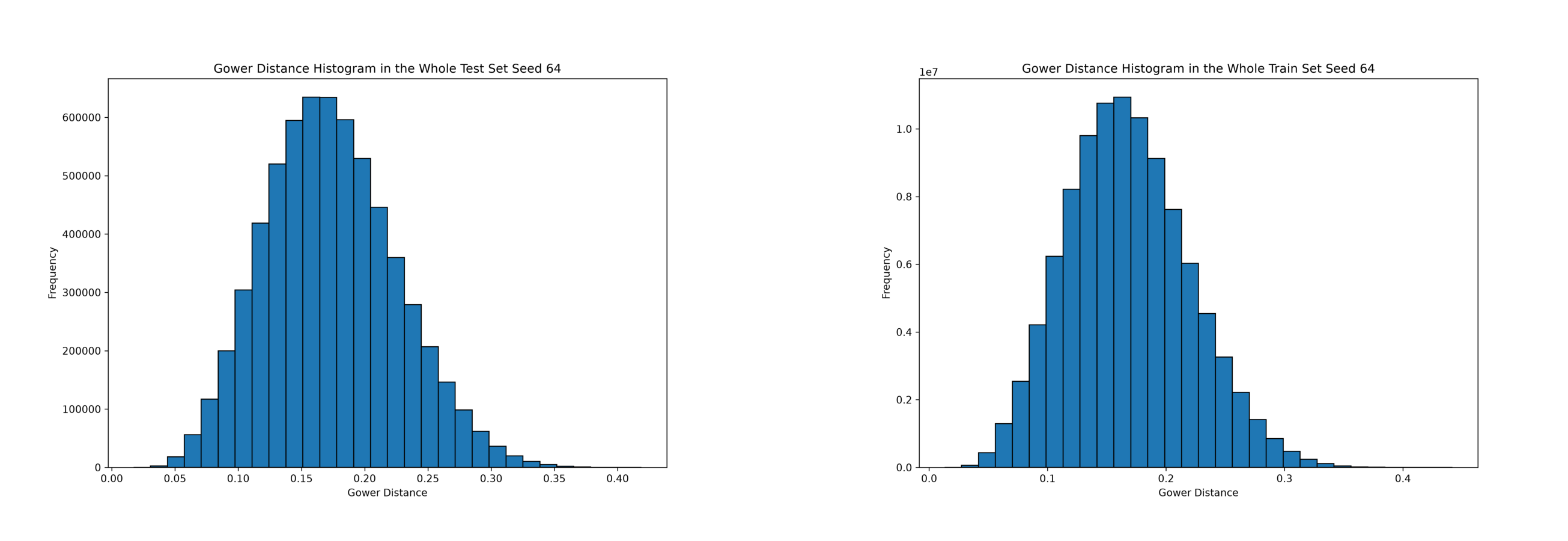}
    \end{minipage}}
    \caption{The histogram of the Gower distance in test set. The chosen seed is 64, corresponding to the maximum AUC-delta parsimonious model. The histograms show that the distribution of Gower distance is close to Poisson or Gaussian distribution, which helps show that both the training set and testing set are non-biased and are not having preference towards specific subgroups which share certain similarities.}
    \label{fig:umap-test1,fig:umap-test2,fig:umap-test3}
\end{figure}


\onecolumn
\begin{landscape}
\begin{longtable}{p{5.2cm}p{2.2cm}ccccc}
\caption{Baseline Demographic and Clinical Characteristics} \label{tab:table1} \\
\toprule
 &  & \multicolumn{5}{r}{Grouped by cvd\_death} \\
 &  & Missing & Overall & False & True & P-Value \\
\midrule
\endfirsthead
\caption[]{Baseline Demographic and Clinical Characteristics} \\
\toprule
 &  & \multicolumn{5}{r}{Grouped by cvd\_death} \\
 &  & Missing & Overall & False & True & P-Value \\
\midrule
\endhead
\midrule
\multicolumn{7}{r}{Continued on next page} \\
\midrule
\endfoot
\bottomrule
\endlastfoot
n &  &  & 17748 & 16992 & 756 &  \\
\addlinespace
Age (years), mean (SD) &  & 0 & 51.5 (13.5) & 50.9 (13.2) & 65.3 (11.0) & $<$0.001 \\
\addlinespace
\multirow[t]{2}{*}{Sex, n (\%)} & female &  & 8989 (50.6) & 8671 (51.0) & 318 (42.1) & $<$0.001 \\
 & male &  & 8759 (49.4) & 8321 (49.0) & 438 (57.9) &  \\
\addlinespace
\multirow[t]{5}{*}{Race / Ethnicity, n (\%)} & Mexican American &  & 3682 (20.7) & 3557 (20.9) & 125 (16.5) & $<$0.001 \\
 & Non-Hispanic Black &  & 3403 (19.2) & 3222 (19.0) & 181 (23.9) &  \\
 & Non-Hispanic White &  & 8719 (49.1) & 8317 (48.9) & 402 (53.2) &  \\
 & Other / Multi-Racial &  & 677 (3.8) & 665 (3.9) & 12 (1.6) &  \\
 & Other Hispanic &  & 1267 (7.1) & 1231 (7.2) & 36 (4.8) &  \\
\addlinespace
\multirow[t]{7}{*}{Education Level, n (\%)} & Less than 9th Grade &  & 2449 (13.8) & 2280 (13.4) & 169 (22.4) & $<$0.001 \\
 & 9th–11th Grade &  & 2737 (15.4) & 2581 (15.2) & 156 (20.6) &  \\
 & High School Grad / GED &  & 4142 (23.3) & 3952 (23.3) & 190 (25.1) &  \\
 & Some College / AA degree &  & 4601 (25.9) & 4462 (26.3) & 139 (18.4) &  \\
 & College Graduate or Above &  & 3798 (21.4) & 3698 (21.8) & 100 (13.2) &  \\
 & Refused &  & 8 (0.0) & 6 (0.0) & 2 (0.3) &  \\
 & Don't Know &  & 13 (0.1) & 13 (0.1) & 0 (0.0) &  \\
\addlinespace
Poverty-Income Ratio, mean (SD) &  & 1378 & 2.8 (1.6) & 2.8 (1.6) & 2.3 (1.5) & $<$0.001 \\
\addlinespace
Systolic BP (mmHg), mean (SD) &  & 0 & 124.7 (17.3) & 124.2 (17.0) & 135.3 (19.3) & $<$0.001 \\
\addlinespace
Diastolic BP (mmHg), mean (SD) &  & 0 & 72.2 (12.3) & 72.3 (12.1) & 69.1 (15.4) & $<$0.001 \\
\addlinespace
\multirow[t]{2}{*}{Hypertension Meds, n (\%)} & False &  & 13467 (75.9) & 13082 (77.0) & 385 (50.9) & $<$0.001 \\
 & True &  & 4281 (24.1) & 3910 (23.0) & 371 (49.1) &  \\
\addlinespace
Total Cholesterol (mg/dL), mean (SD) &  & 0 & 204.4 (39.1) & 204.4 (39.0) & 204.7 (41.0) & 0.846 \\
\addlinespace
HDL Cholesterol (mg/dL), mean (SD) &  & 0 & 52.6 (15.6) & 52.6 (15.6) & 51.9 (15.7) & 0.235 \\
\addlinespace
BMI (kg/m²), mean (SD) &  & 0 & 28.7 (5.5) & 28.7 (5.5) & 29.2 (5.7) & 0.035 \\
\addlinespace
Height (cm), mean (SD) &  & 0 & 167.8 (10.1) & 167.8 (10.1) & 167.4 (10.1) & 0.358 \\
\addlinespace
Waist Circumference (cm), mean (SD) &  & 252 & 99.0 (15.1) & 98.8 (15.0) & 103.3 (15.8) & $<$0.001 \\
\addlinespace
\multirow[t]{2}{*}{Diabetes Mellitus, n (\%)} & False &  & 15887 (89.5) & 15311 (90.1) & 576 (76.2) & $<$0.001 \\
 & True &  & 1861 (10.5) & 1681 (9.9) & 180 (23.8) &  \\
\addlinespace
\multirow[t]{2}{*}{Current Smoker, n (\%)} & False &  & 13777 (77.6) & 13210 (77.7) & 567 (75.0) & 0.084 \\
 & True &  & 3971 (22.4) & 3782 (22.3) & 189 (25.0) &  \\
\addlinespace
\multirow[t]{2}{*}{Statin Use, n (\%)} & False &  & 15470 (87.2) & 14895 (87.7) & 575 (76.1) & $<$0.001 \\
 & True &  & 2278 (12.8) & 2097 (12.3) & 181 (23.9) &  \\
\addlinespace
Serum Creatinine (mg/dL), mean (SD) &  & 0 & 0.9 (0.4) & 0.9 (0.4) & 1.0 (0.8) & $<$0.001 \\
\addlinespace
eGFR (mL/min/1.73 m²), mean (SD) &  & 0 & 94.6 (19.1) & 95.2 (18.8) & 81.4 (21.9) & $<$0.001 \\
\addlinespace
HbA1c (\%), median [Q1,Q3] &  & 35 & 5.5 [5.2,5.8] & 5.5 [5.2,5.8] & 5.7 [5.4,6.3] & $<$0.001 \\
\addlinespace
ALT (U/L), mean (SD) &  & 43 & 26.9 (29.1) & 27.0 (29.4) & 24.6 (21.3) & 0.003 \\
\addlinespace
AST (U/L), mean (SD) &  & 43 & 26.3 (22.1) & 26.3 (22.3) & 26.5 (18.4) & 0.842 \\
\addlinespace
Albumin (g/dL), mean (SD) &  & 0 & 42.6 (3.2) & 42.6 (3.2) & 41.9 (3.2) & $<$0.001 \\
\addlinespace
Hemoglobin (g/dL), mean (SD) &  & 32 & 14.3 (1.5) & 14.3 (1.5) & 14.2 (1.6) & 0.109 \\
\addlinespace
Hematocrit (\%), mean (SD) &  & 32 & 42.1 (4.3) & 42.1 (4.3) & 42.2 (4.5) & 0.797 \\
\addlinespace
WBC Count (1000 cells/µL), mean (SD) &  & 32 & 7.1 (2.4) & 7.1 (2.4) & 7.3 (2.2) & 0.043 \\
\addlinespace
Platelet Count (1000 cells/µL), mean (SD) &  & 32 & 264.0 (68.3) & 264.4 (68.0) & 254.2 (73.4) & $<$0.001 \\
\addlinespace
CRP (mg/L), median [Q1,Q3] &  & 6 & 2.1 [0.9,4.8] & 2.1 [0.9,4.7] & 3.0 [1.4,6.5] & $<$0.001 \\
\addlinespace
\multirow[t]{2}{*}{History of Cancer, n (\%)} & False &  & 16270 (91.7) & 15601 (91.8) & 669 (88.5) & 0.002 \\
 & True &  & 1478 (8.3) & 1391 (8.2) & 87 (11.5) &  \\
\addlinespace
\end{longtable}
\end{landscape}

\begin{landscape}
\begin{longtable}[l]{@{} p{5.5cm} p{2.2cm} c c >{\centering\arraybackslash}p{4cm} >{\centering\arraybackslash}p{4cm} c @{}}
\caption{Baseline Demographic and Clinical Characteristics} \label{tab:table2} \\
\toprule
 &  & \multicolumn{5}{r}{Grouped by Outcome} \\
 &  & Missing & Overall & High-Performance Subgroup & Low-Performance Subgroup & P-Value \\
\midrule
\endfirsthead
\caption[]{Baseline Demographic and Clinical Characteristics} \\
\toprule
 &  & \multicolumn{5}{r}{Grouped by Outcome} \\
 &  & Missing & Overall & High-Performance Subgroup & Low-Performance Subgroup & P-Value \\
\midrule
\endhead
\midrule
\multicolumn{7}{r}{Continued on next page} \\
\midrule
\endfoot
\bottomrule
\endlastfoot
n &  &  & 3550 & 1775 & 1775 &  \\
\addlinespace
\multirow[t]{2}{*}{Hypertension Meds, n (\%)} & False &  & 2698 (76.0) & 1608 (90.6) & 1090 (61.4) & $<$0.001 \\
 & True &  & 852 (24.0) & 167 (9.4) & 685 (38.6) &  \\
\addlinespace
\multirow[t]{2}{*}{Statin Use, n (\%)} & False &  & 3099 (87.3) & 1692 (95.3) & 1407 (79.3) & $<$0.001 \\
 & True &  & 451 (12.7) & 83 (4.7) & 368 (20.7) &  \\
\addlinespace
\multirow[t]{2}{*}{Current Smoker, n (\%)} & False &  & 2786 (78.5) & 1600 (90.1) & 1186 (66.8) & $<$0.001 \\
 & True &  & 764 (21.5) & 175 (9.9) & 589 (33.2) &  \\
\addlinespace
\multirow[t]{2}{*}{Diabetes Mellitus, n (\%)} & False &  & 3162 (89.1) & 1710 (96.3) & 1452 (81.8) & $<$0.001 \\
 & True &  & 388 (10.9) & 65 (3.7) & 323 (18.2) &  \\
\addlinespace
Age (years), mean (SD) &  & 0 & 51.4 (13.5) & 46.3 (12.2) & 56.5 (12.9) & $<$0.001 \\
\addlinespace
\multirow[t]{2}{*}{History of Cancer, n (\%)} & False &  & 3291 (92.7) & 1751 (98.6) & 1540 (86.8) & $<$0.001 \\
 & True &  & 259 (7.3) & 24 (1.4) & 235 (13.2) &  \\
\addlinespace
\multirow[t]{6}{*}{Education Level, n (\%)} & Less than 9th Grade &  & 482 (13.6) & 153 (8.6) & 329 (18.5) & $<$0.001 \\
 & 9th–11th Grade &  & 563 (15.9) & 180 (10.1) & 383 (21.6) &  \\
 & High School Grad / GED &  & 831 (23.4) & 373 (21.0) & 458 (25.8) &  \\
 & Some College / AA degree &  & 887 (25.0) & 496 (27.9) & 391 (22.0) &  \\
 & College Graduate or Above &  & 784 (22.1) & 570 (32.1) & 214 (12.1) &  \\
 & Don't Know &  & 3 (0.1) & 3 (0.2) & 0 (0.0) &  \\
\addlinespace
\multirow[t]{2}{*}{Sex, n (\%)} & Female &  & 1820 (51.3) & 936 (52.7) & 884 (49.8) & 0.087 \\
 & Male &  & 1730 (48.7) & 839 (47.3) & 891 (50.2) &  \\
\addlinespace
Poverty-Income Ratio, mean (SD) &  & 0 & 2.7 (1.6) & 3.3 (1.6) & 2.2 (1.5) & $<$0.001 \\
\addlinespace
Systolic BP (mmHg), mean (SD) &  & 0 & 124.5 (17.6) & 119.3 (15.7) & 129.8 (17.9) & $<$0.001 \\
\addlinespace
\multirow[t]{5}{*}{Race / Ethnicity, n (\%)} & Mexican American &  & 718 (20.2) & 344 (19.4) & 374 (21.1) & $<$0.001 \\
 & Non-Hispanic Black &  & 722 (20.3) & 240 (13.5) & 482 (27.2) &  \\
 & Non-Hispanic White &  & 1734 (48.8) & 997 (56.2) & 737 (41.5) &  \\
 & Other / Multi-Racial &  & 128 (3.6) & 41 (2.3) & 87 (4.9) &  \\
 & Other Hispanic &  & 248 (7.0) & 153 (8.6) & 95 (5.4) &  \\
\addlinespace
Diastolic BP (mmHg), mean (SD) &  & 0 & 72.0 (11.9) & 71.4 (10.4) & 72.6 (13.2) & 0.003 \\
\addlinespace
HDL Cholesterol (mg/dL), mean (SD) &  & 0 & 52.8 (15.6) & 52.6 (14.4) & 53.1 (16.7) & 0.382 \\
\addlinespace
Total Cholesterol (mg/dL), mean (SD) &  & 0 & 205.1 (39.5) & 205.4 (38.5) & 204.7 (40.4) & 0.592 \\
\addlinespace
Hemoglobin (g/dL), mean (SD) &  & 0 & 14.3 (1.5) & 14.3 (1.4) & 14.3 (1.6) & 0.909 \\
\addlinespace
Hematocrit (\%), mean (SD) &  & 0 & 42.1 (4.3) & 42.0 (4.1) & 42.1 (4.5) & 0.577 \\
\addlinespace
\multirow[t]{2}{*}{cvd\_death, n (\%)} & False &  & 3399 (95.7) & 1743 (98.2) & 1656 (93.3) & $<$0.001 \\
 & True &  & 151 (4.3) & 32 (1.8) & 119 (6.7) &  \\
\addlinespace
weighted AUC, median [Q1,Q3] &  & 0 & 0.8 [0.7,0.9] & 0.9 [0.9,0.9] & 0.7 [0.7,0.8] & $<$0.001 \\
\addlinespace
\end{longtable}
\end{landscape}

\printbibliography
\end{document}